\documentclass[11pt,twocolumn]{article}

\usepackage[numbers]{natbib}

\usepackage{amsmath,amsthm,amssymb,amsfonts}
\usepackage{bm}

\usepackage{booktabs} 
\usepackage{threeparttable}
\usepackage{makecell}
\usepackage{multirow}
\usepackage{placeins}

\usepackage{lmodern}
\usepackage{subcaption}
\usepackage{geometry}
 \usepackage[nice]{nicefrac}
\usepackage{mathabx}

\newcommand{\todo}[2][]{ }

\newcommand{\jj}[1]{ }
\newcommand{\me}[1]{ }
\newcommand{\dovid}[1]{ }
\newcommand{\uri}[1]{ }
\usepackage{tcolorbox}
\usepackage{pifont}
\definecolor{mydarkgreen}{RGB}{39,130,67}
\definecolor{mydarkred}{RGB}{192,25,25}
\definecolor{mydarkblue}{RGB}{0,0,140}

 \usepackage{empheq}
 \usepackage{color}
\definecolor{darkgreen}{rgb}{0.00,0.5,0.00}

\usepackage{hyperref}
\hypersetup{
    pagebackref=true,
    backref=true
	colorlinks=true,        
	linkcolor=blue,         
	citecolor=mydarkblue,         
	urlcolor=blue           
}

\usepackage{natbib}
\usepackage{tcolorbox}
\usepackage{framed}
\usepackage{algorithm}
\usepackage{algpseudocode}
\usepackage{enumerate}
\usepackage{cleveref}
\newtheorem{theorem}{Theorem}
\newtheorem{assumption}{Assumption}
\newtheorem{definition}{Definition}
\newtheorem{lemma}{Lemma}
\newtheorem{proposition}{Proposition}

\crefname{assumption}{assumption}{assumptions}
\usepackage{graphicx}
\usepackage{float}

\theoremstyle{definition}

\usepackage{sidecap}

\providecommand{\jj}{\mathbf{j}}

\newcommand{\CPTE}{{\mathrm{CPTE}}}

\newcommand{\indep}{\perp \!\!\! \perp}

\newcommand{\cY}{\mathcal{Y}}

\newcommand{\cI}{\mathcal{I}}

\newcommand{\cO}{\mathcal O}

\newcommand{\R}{\mathbb R}

\newcommand{\eqdef}{\stackrel{\text{def}}{=}}

\def\<#1,#2>{\langle #1,#2\rangle}

\usepackage{dsfont}

\renewcommand{\leq}{\leqslant}
\renewcommand{\geq}{\geqslant}
\renewcommand{\le}{\leqslant}

\def\<{\langle}
\def\>{\rangle}
\def\|{\Vert}

\def\eps{\varepsilon}
\def\var{{\rm var\,}}

\newcommand{\esp}[1]{\mathbb{E}\left[#1\right]}

\newcommand{\NRM}[1]{{{\left\| #1\right\|}}} 
\newcommand{\set}[1]{{{\left\{ #1\right\}}}} 
\newcommand{\proba}[1]{\mathbb{P}\left(#1\right)}

\renewcommand{\P}{\mathbb{P}}

\newcommand{\cN}{\mathcal{N}}
\newcommand{\cX}{\mathcal{X}}

\newcommand{\dd}{{\rm d}}

\newcommand{\cP}{\mathcal{P}}

\newcommand{\cQ}{\mathcal{Q}}

\newcommand{\cD}{\mathcal{D}}
\newcommand{\one}{\mathds{1}}

\usepackage{amsfonts}
\usepackage{amsthm}
\usepackage{amsmath}
\usepackage{amssymb}
\usepackage{mathrsfs}
\usepackage{amscd}
\usepackage{caption}
\usepackage{indentfirst}
\usepackage{cleveref}

\DeclareMathOperator*{\argmax}{argmax}

\newcommand{\cov}{{\rm cov}}

\title{Preference-based Conditional Treatment Effects and Policy Learning}
\author{%
Dovid Parnas\thanks{Equal contribution.}\\
Technion -- Israel Institute of Technology
\and
Mathieu Even\footnotemark[1]\\
Inria, Inserm, Universit\'e de Montpellier, France
\and
Julie Josse\\
Inria, Inserm, Universit\'e de Montpellier, France
\and
Uri Shalit\\
Department of Statistics and Operations Research, Tel Aviv University
}

\begin{document}
\maketitle

\begin{abstract}
We introduce a new preference-based framework for conditional treatment effect estimation and policy learning, built on the Conditional Preference-based Treatment Effect (CPTE). CPTE requires only that outcomes be ranked under a preference rule, unlocking flexible modeling of heterogeneous effects with multivariate, ordinal, or preference-driven outcomes. This unifies applications such as conditional probability of necessity and sufficiency, conditional Win Ratio, and Generalized Pairwise Comparisons. Despite the intrinsic non-identifiability of comparison-based estimands, CPTE provides interpretable targets and delivers new identifiability conditions for previous unidentifiable estimands. We present estimation strategies via matching, quantile, and distributional regression, and further design efficient influence-function estimators to correct plug-in bias and maximize policy value. Synthetic and semi-synthetic experiments demonstrate clear performance gains and practical impact.

\end{abstract}






\section{Introduction}




In high-stakes domains, such as healthcare and public policy, clinicians and policymakers must make decisions about treatments whose effects can vary widely between individuals. Population-level averages are rarely enough: the right choice for each person depends on their specific characteristics, risk factors and likely response. To truly personalize data-driven decisions, we must estimate individual effects and learn feature-dependent policies.

The standard approach estimates conditional treatment effects, contrasting conditional outcomes for each treatment arm, and derives policies accordingly. However, while the conditional average treatment effect (CATE) has become the default tool for heterogeneous effect estimation and policy learning, it can be misleading. CATE-based policies rely on averages, ignoring complex treatment responses. For example, a heavy tailed distribution can make CATE-based policies favor treatments where the majority actually loses.
Consider salary as an outcome and an intervention that causes small monthly losses for many low-salary individuals but large gains for a handful of high-salary individuals. 
A perfectly estimated CATE would recommend the intervention, despite diminishing the welfare of the majority. In critical settings, this bias is not just a statistical artifact, it threatens fair, effective decision-making.

Moreover, defining a CATE policy is not straightforward for multiple outcomes. Consider heart-failure studies \citep{Pocock2011winratio} where death, stroke, and hospitalization are all outcomes of primary interest. These outcomes follow a natural hierarchy: death being the most severe, followed by stroke and then hospitalization. The challenge lies in incorporating this hierarchy into a principled, data-driven framework for decision making.
A leading approach over the past decade is the \textit{Win Ratio} \citep{Pocock2011winratio}, where treated and control patients are paired and compared sequentially. Death is assessed first; stroke is compared if both survived; if neither patient experienced a stroke, hospitalization is evaluated. This design explicitly respects clinical priorities.
However, quantifying treatment effects, conditional effects, and learning policies when outcomes can only be judged through comparisons remains highly non-trivial \citep{mueller2023personalized,even2025rethinking,tian2000probabilities}. They cannot be captured by CATE. The core difficulty is the non-identifiability of the underlying estimands; it is impossible to estimate them from observable data.
due to the \textit{fundamental problem of causal inference} \citep{Holland1986}. Indeed, access to the \textit{joint} potential outcomes distribution is required to compare counterfactuals. However, no unit can ever be observed in both the treatment and control groups, even in randomized studies.

This paper introduces a general framework for learning preference-based conditional treatment effects and policies from both randomized and observational data. Preference-based learning directly answers the question, "Which treatment should be preferred?" It also expands policy treatment effect estimation and policy learning to outcomes that are comparable but not quantifiable.
A first example is the probability of necessity and sufficiency (PNS), a classic causal quantity \citep{mueller2023personalized,tian2000probabilities}. PNS can be interpreted as the probability that the treated outcome is strictly preferred to the control outcome for a given individual, averaged across the population. Although related directions have been explored \citep{wang2025estimating,wang2025learning,shingaki2021identification}, it lacks an identifiable, conditional construct tailored to personalized decision-making.

More broadly, preference-based effects extend naturally to complex, multivariate outcomes whenever a rule specifies which outcome is preferred in a pairwise comparison. This framework subsumes hierarchical outcome settings such as the Win Ratio and Generalized Pairwise Comparisons \citep{Pocock2011winratio,Buyse2010}. Our method goes beyond these cases, learning conditional effects and optimal policies under such preference structures, while also accommodating more general, possibly non-strict, preference rules.

\textbf{Related work.}
Most work on conditional treatment effects and policy learning in causal inference focuses on single real-valued or binary outcomes \citep{AtheyImbens2016,WagerAtheyTibshirani2019,kunzel2019metalearners,athey2021policy}.
This paper offers a broader framework, connecting three related lines of work: \textit{(i)} non-identifiable causal quantities such as the PNS and their bounds, \textit{(ii)} pairwise comparison methods like the Win Ratio and Net Benefit, and \textit{(iii)} distributional regression and quantile treatment effects.

The \textit{Probability of Necessity and Sufficiency} (\textit{PNS}) is the probability that treatment improves the outcome and is required to do so \citep{mueller2023personalized,tian2000probabilities}. It can be viewed as the population average of individual-level effects, with subpopulation extensions studied in \citet{kawakami2024probabilities,wang2025estimating}. \citet{lei2024causal} presents \textit{population}-level point estimates of PNS. However, to our knowledge, a clear definition of an identifiable, conditional PNS is missing.

Because PNS and its conditional variants are fundamentally \textit{non-identifiable}, most prior work has focused on deriving identifiable \textit{bounds} from data under suitable assumptions \citep{shingaki2021identification,robins2010alternative,tian2000probabilities}. Bounding strategies have been developed for binary \citep{mueller2023personalized}, ordinal \citep{huang2017inequality,lu2018treatment, li2024probabilities, de2025probability}, continuous \citep{manski2003partial}, quantile treatment effects (QTE) \citep{koenker1978regression,firpo2007efficient} and conditional QTE \citep{abadie2002instrumental,firpo2007efficient,cui2023policy}.
Instead of relying on identifiable bounds, which can be loose for conditional quantities or non-binary outcomes \citep{li2024probabilities,manski2003partial}, we introduce and estimate an interpretable, identifiable proxy.

\citet{Pocock2011winratio} and \citet{Buyse2010} introduced the \textit{Win Ratio} and \textit{Net Benefit} to compare two treatments over multiple, \textit{hierarchical} outcomes. These metrics allow several outcomes to factor into the evaluation of treatment effects and decision making.
\citet{petit2025optimal} considered optimal treatment rules for the net benefit. However, our broader framework goes further by defining conditional treatment effect with hierarchical outcomes, enabling policy value optimization and developing semi-parametric efficient estimators.
Our work borrows the formalism of \citep{even2025rethinking} and non-causal Probabilistic Index Models (PIMs, \citep{thas2012probabilistic}).


Our estimation methods are based on nearest neighbor approaches \citep{biau2015lectures} and off-the-shelf quantile regression methods \citep{Koenker1978,greene2003econometric,meinshausen2006quantile} for nuisance parameters. Our methods differ from distributional treatment effect approaches \citep{naf2024causal,zhou2022estimating,pmlr-v139-park21c,kallus2023robust}: we use distribution methods \textit{as a means} to estimate preference-based effects.

\textbf{Outline of the paper and contributions.}
We first introduce the causal and preference-based setting in \Cref{sec:setting} and define our target estimands in \Cref{sec:conditional_effects}. We distinguish between a non-identifiable conditional Individual Treatment Effect (ITE) and our new identifiable estimand, the Conditional Preference Treatment Effect (CPTE), and establish conditions under which they coincide via the concept of \textit{structural treatment effect modifiers}.
In \Cref{sec:CPTE_estimation}, we develop estimation strategies for CPTE, using flexible building blocks like matching and distributional regression. In \Cref{sec:policy}, we turn to preference-based policy learning, introducing a plug-in method and a policy value optimization approach. Finally, in \Cref{sec:semi_param_efficient}, we derive a semiparametric efficient policy estimator based on a one-step efficient influence function (EIF) correction.
We validate these methods on synthetic and semi-synthetic experiments (\Cref{sec:exps}, \Cref{app:exps}).
Our contribution is threefold, enabling practical policy learning within a preference-based framework. We build an identifiable conditional proxy, the CPTE, for individual-level preference and establish its identifiability conditions. We develop an estimation procedure for the CPTE. Finally, we learn decision policies based on preference rules. 

\section{Setting}\label{sec:setting}

\textbf{Causal assumptions.}
We assume access to $n$ i.i.d.\ samples, associated with feature vectors $X_i\in\cX$, treatment $T_i\in\{0,1\}$, and observed outcome $Y_i\in\cY$ (possibly multivariate, e.g.\ $\cY\subset\R^d$). Under the \emph{potential outcome} framework \citep{splawa1990application}, each unit has two outcomes, $Y_i(0),Y_i(1)\in\cY$, only one of which is observed.

We assume \textit{SUTVA}, \emph{unconfoundedness} and \emph{positivity}:
\begin{assumption}[SUTVA]\label{hyp:sutva}
$Y_i = Y_i(T_i)$ for all $i\in[n]$.
\end{assumption}
\begin{assumption}[unconfoundedness]\label{hyp:unconfoundedness}
$\{Y_i(0),Y_i(1)\} \perp T_i \mid X_i$.
\end{assumption}

\begin{assumption}\label{hyp:positivity}
There exists $\eta\in(0,1)$ such that $\eta \le e(x) \le 1-\eta$ for all $x\in\cX$,
where $e(x) = \proba{T_i=1 \mid X_i=x}$.
\end{assumption}

\textbf{Preference-based framework.}
We ask whether a unit would fare better under treatment than without. Given two outcomes $y,y'\in\cY$, ``better'' cannot always be quantified correctly with the sign and magnitude of their linear difference $y-y'$. Instead, we consider a more general ``preference function'':

\begin{definition}[Preference function]\label{win_function_def}
$w:(y,y')\in\cY^2 \mapsto w(y|y') \in \R$
quantifies how $y$ fares against $y'$.
\end{definition}

Typical examples of preference functions are: \\
\textit{(i)} \textit{PNS}, For single real-valued outcomes:
$$
w(y|y')=\mathbf 1\{y>y'\}\,.
$$
\textit{(ii)} \textit{General Ordered Outcomes} like the Win Ratio:
\begin{equation}
    \label{eq:win_lexico}
    w(y|y')= \one_\set{y\succ y'} + 0.5\cdot\one_\set{y\sim y'}\,,
\end{equation}
where $\succ$ is an order on $\cY$, such as the \textit{lexicographic order} on $\R^d$:
$y\succ y'$ iff $\inf\{k:y_k>y_k'\}<\inf\{k:y_k<y_k'\}$ used for hierarchical outcomes \citep{Pocock2011winratio}.

\textit{(iii)} \textit{Unknown preferences}: If an “ideal stochastic oracle’’ compares the outcomes $Y$ and $Y'$ of two units,
$$
w(Y|Y')\eqdef \P(Y\succ Y')+\tfrac12\P(Y\sim Y'),
$$
where probabilities capture uncertainty or individual preferences. In this paper, $w$ is assumed to be known.


\section{Conditional Effects}
\label{sec:conditional_effects}

\subsection{Definition of a preference-based conditional treatment effect}
\label{sec:CPTE_def}

Personalized treatment is the holy grail of causal inference. While population-level effects inform broad policy decisions, they overlook the substantial heterogeneity that individualized policies can account for.
However, the Individual Treatment Effect (ITE) is inaccessible due to the \textit{fundamental problem of causal inference}.  
In our preference-based framework with preference function $w$, the ITE for patient $i$ is:
\begin{equation*}
    \Delta_i = w(Y_i(1)|Y_i(0))\,.
\end{equation*}
It measures how $Y_i(1)$ compares to $Y_i(0)$; treatment follows $\Delta_i$.  
Because only one action is observed, the ITE is in general unidentifiable and cannot be estimated from data. Common surrogates for ITE are \textit{conditional treatment effects}, which in many cases can be estimated from data.
Conditional effects answer subgroup causal questions and support policy learning \citep{athey2021policy}, ideally balancing identifiability and the ability to tailor treatments to specific populations. 
For the risk difference (RD) $w(y|y')=y-y'$, the widely used Conditional Average Treatment Effect (CATE) is exactly the conditional mean of the ITE:
\[
\tau_\mathrm{RD}(x)=\esp{Y_i(1)-Y_i(0)|X_i=x}\,.
\]
It is identifiable under \Cref{hyp:positivity,hyp:sutva,hyp:unconfoundedness} and estimable via a variety of methods \citep{nie2021quasi,wager2018estimation,kunzel2019metalearners}.  
Unfortunately, for the more generalized preference functions considered in our paper, we have the following unidentifiability result.

\begin{lemma}\label{lem:unidentifiable}
The conditional ITE, defined as: \begin{equation}\label{eq:unidentifiable_CPTE}
    \esp{\Delta_i|X_i=x} = \esp{w(Y_i(1)|Y_i(0))\,|\,X_i=x}\,,
\end{equation}
 is generally unidentifiable and cannot be inferred from the distribution of the observations $(X_i,A_i,Y_i)_{i\in[n]}$.
\end{lemma}

This stems from the non-separability of $w$, e.g. $w(y|y')\neq g(y)-g(y')$, see proof in \Cref{app:proof_non_ident}.
When $w$ is separable and can be written as $w(y|y')= g(y)-g(y')$ for some function $g$, this quantity is the classical CATE with $g(Y_i)$ as outcome, and is identifiable. Note that the non-identifiability in \Cref{lem:unidentifiable} does \textit{not} stem from confounding and holds even in RCT settings.
\citet{mueller2023personalized,cui2023policy,li2024probabilities} explore identifiable bounds on the conditional ITE for the PNS and QTE.

We propose overcoming the non-identifiability of the non-separable, conditional ITE by defining an \textit{identifiable} conditional effect.
We call our new estimand the Conditional Preference Treatment Effect (CPTE). It is obtained by conditioning on covariates and replacing the joint distribution of counterfactuals, $(Y_i(1),Y_i(0))|X_i=x$, with $(Y_i(1),Y_j(0))|(X_i,X_j)=(x,x)$, drawing inspiration from the formalism of probabilistic index models \citep{thas2012probabilistic}.
In the next subsection, we introduce identifiability conditions under which the CPTE will be equal to the Conditional ITE. 

\begin{definition}[Preference-based conditional effect]
\label{def:CATE}
For $x\in\cX$, let $\set{Y(0,x),Y(1, x)}$ be an independent copy of $\set{Y_i(0),Y_i(1)}|X_i=x$.  
Define
\begin{equation}\label{eq:cpte}
    \begin{aligned}
    	\CPTE(x)&\eqdef\esp{w(Y(1, x)|Y_i(0))\,|\,X_i=x}\\
        &=\esp{w(Y_i(1)|Y(0,x)\,|\,X_i=x}\,,
    \end{aligned} 
\end{equation}
the Conditional Preference Treatment Effect.  
Equivalently, for $i\ne j$,
\[
\CPTE(x)=\esp{w(Y_i(1)|Y_j(0))|(X_i,X_j)=(x,x)}\,,
\]
where the expectation is over the products of the probability distributions $Y_i(1)|X_i=x$ and $Y_j(0)|X_j=x$.
\end{definition}

Under \Cref{hyp:positivity,hyp:sutva,hyp:unconfoundedness},
\[
\CPTE(x)=
\esp{w(Y_i|Y_j)\,|\,X_i=X_j=x,\,T_i=1,T_j=0}\,
\]
Therefore, it is \textit{identifiable}.  
A population effect can be defined as
$\esp{\CPTE(X_i)}$, which is \textit{directly collapsible} (written as the population average of conditional effects \citep{even2025rethinking,colnet2023risk}).
Note that because an independent copy always exists, we do not require further assumptions for the definition of $\CPTE(x)$, unlike the net benefit as defined by \citet{petit2025optimal}.
Finally, we define treatment effect modifiers as follows:
\begin{definition}\label{def:effect_modif}
A variable $X^k$ is a treatment effect modifier if $u\mapsto \esp{\CPTE(X_i)|X_{i}^k=u, X_i^{-k}=v}$ is not constant when conditioning on other features $X^{-k}$.
Effects are homogeneous when no modifier exists, and heterogeneous otherwise.
\end{definition}

Treatment-effect modifiers depend on the effect measure $w$ and may differ from those of the classical risk difference. 
Next, we explain how our CPTE relates to the non-identifiable Conditional ITE $\esp{\Delta_i|X_i=x}$.
In particular, the next subsection provides new identifiability conditions under which $\CPTE(x)=\esp{\Delta_i|X_i=x}$, as well as tools to assess how far $\CPTE(x)$ is from the non-identifiable Conditional ITE when these conditions do not hold.\\

\subsection{Properties of $\CPTE(x)$ and dependencies over features}
\label{sec:CPTE_properties}

\textit{Which covariates should be included in $X_i$?}
In order for $\CPTE(x)$ to be identifiable, the classic causal assumptions (\Cref{hyp:positivity,hyp:sutva,hyp:unconfoundedness}) must hold.
In particular, \Cref{hyp:unconfoundedness} requires that all confounders must be included as covariates.
\textit{Treatment effect modifiers} (\Cref{def:effect_modif}) are features that have an influence on the treatment effect. However, they are not necessarily confounders, that is, variables that also influence treatment assignment. Nevertheless, $\CPTE(x)$ requires them to capture population heterogeneity, as is often the case for conditional effects.

To achieve identifiability of the Conditional ITE, defined in \Cref{eq:unidentifiable_CPTE}, \textit{via} the equality $\CPTE(x)=\esp{\Delta_i|X_i=x}$, treatment effect modifiers and confounders are not enough, as highlighted in \Cref{lem:unidentifiable}.
Bearing this in mind, we introduce the notion of \textit{structural treatment effect modifiers}. These are all the variables that could make the outcome of a treated patient better or worse than that of a control, from the set of variables $Z_i$ that spans all endogenous and exogenous variables in the structural causal model, as introduced by \citet[Chapter 1]{Pearl2009-av}.
This definition differs from that of \Cref{def:effect_modif}, which only considers \textit{observed} variables.

\begin{definition}[Structural treatment effect modifiers]
\label{def:structural_treatment_effect_modifiers}
    Structural treatment effect modifiers are a minimal set of variables $\cI\subset[p]$ for which there exists a non-constant function $g$ such that for all $i,j\in[n]$ and all $x\in\R^\cI$:
    \begin{equation*}
        \esp{w(Y_i(1),Y_j(0))\,|\, Z_i,Z_j,(Z_i^\cI,Z_j^\cI)=(x,x)}=g(x)\,,
    \end{equation*}
    where $z^\cI=(z_k)_{k\in\cI}$ for $z\in\R^p$.
\end{definition}

As the next lemma shows, if all structural treatment effect modifiers are included in the covariates, then we have equality between $\CPTE(x)$ and $\esp{w(Y_i(1)|Y_i(0))|X_i=x}$. 
However, if this does not hold, $\CPTE(x)$ can function as an identifiable proxy for $\esp{w(Y_i(1)|Y_i(0))|X_i=x}$, requiring \textit{sensitivity analyses} to help assess how far we are from equality and identifiability \citep{nguyen2018sensitivity,colnet2022causal,jin2023sensitivity}.
\begin{theorem}
\label{lem:treatment_effect_modifiers}
    Assume that structural treatment effect modifiers in the sense of \Cref{def:structural_treatment_effect_modifiers} are included in the set of features $X_i$. Then, $\CPTE(x)=\esp{w(Y_i(1)|Y_i(0))|X_i=x}$ for all $x\in\cX$.
\end{theorem}

In summary, features $X_i$ need to include \textit{confounding variables} for identifiability of the CPTE (in the case of an observational study), \textit{treatment effect modifiers} to capture treatment effect heterogeneity, and \textit{structural treatment effect modifiers} for the equality in \Cref{lem:treatment_effect_modifiers} to hold. However, it might not always be possible to include, or even know, all structural treatment effect modifiers. For example, in the medical context, these can come from genetic factors or other unmeasured sources of variation. Yet, as we show in the experiments below, even in such cases the CPTE can be a useful estimand.
We note that another assumption which leads to the same identifiability result is \textit{potential independence}, namely that $Y_i(1)\indep Y_i(0)|X_i$ for all $i$.
This assumption, stating that the potential outcome under control tells nothing about the potential outcome under treatment, is unlikely to hold.\\
A presentation of the above definitions and Theorem \ref{lem:treatment_effect_modifiers} in SCM terminology is available in the Appendix \ref{app:SCM}.


\subsection{Examples}
\label{sec:application_examples}
\textbf{Preference conditional PNS.}
For $\cY\subset \R$ and $w(y|y')=\one_\set{y>y'}$, we obtain a conditional and identifiable PNS that we call \textit{preference conditional PNS}. This contrasts with the classical PNS \citep{tian2000probabilities} and its conditional counterparts, which are unidentifiable.
Conditioning on covariates and comparing treated and untreated outcomes across independent copies yields an identifiable estimand, preserving the causal meaning of the PNS and enabling personalized policies.
With \textit{binary outcomes}, since a Bernoulli distribution is determined by its mean,
$\CPTE(x)$ is a direct function of $\esp{Y_i|X_i=x,A_i=a}$, estimable via standard machine learning.
However, for general real-valued outcomes or for the non-identifiable PNS, no closed form exists. In \Cref{sec:CPTE_estimation} we show how distributional regression methods can be used to estimate the CPTE in this case.

\textbf{Conditional effects with hierarchical outcomes.}
When multiple outcomes, $\cY$, have a hierarchical order $\succ$ (e.g.\ $\cY\subset\R^d$ with a lexicographic order) and $w(y|y')=\one_\set{y\succ y'}$, treatment effects are often estimated by Win Ratio or pairwise comparisons. These usually target the population-level measure $\tau_\mathrm{pop}=\esp{w(Y_i(1)| Y_j(0))}$ \citep{Mao2017} for independent $i,j$. Such measures quantify the preference function of a treated unit against a control unit when both are sampled at random from the population distribution.
By conditioning on $(X_i,X_j)=(x,x)$, $\CPTE(x)$ becomes a personalized version, facilitating policy learning for complex outcomes. 

\textbf{Conditional effects with unknown preferences.}
When preferences are unknown, $w$ must be learned.  
A common approach posits a latent reward $R$ with $w(y|y')=\esp{\exp(R(y)-R(y'))}$, estimated from expert or participant feedback.  
One may also allow $R$ to depend on $X$, yielding unit-specific preferences.  
Then
$\CPTE(x)=\esp{e^{R(Y_i(1))}|X_i=x}\esp{e^{-R(Y_i(0))}|X_i=x},$
where $w$ (and $R$) is assumed learned.

\section{CPTE Estimation Methods}
\label{sec:CPTE_estimation}

We now turn to estimating $x\mapsto\CPTE(x)$ defined in \Cref{def:CATE}. CPTE captures more than first-moment information of the conditional distributions of the counterfactuals ($\esp{Y_i(t)|X_i=x}$). Therefore, our estimation methods learn the conditional distributions $Y_i(1)|X_i$ and $Y_i(0)|X_i$.
If we know how to sample from $\cP_{Y_i|T_i,X_i}$, a consistent and asymptotically normal estimator of $\CPTE(x)$ would be:
\begin{equation}\label{eq:gen_estimation}
    \frac{1}{m}\sum_{k=1}^m w(\hat y_1^{(k)},\hat y_0^{(k)})\,,
\end{equation}
where $(\hat y_t^{(k)})_k$ are independent random variables sampled from $\cP_{Y_i|T_i=t,X_i=x}$, with $(\hat y_1^{(k)})\indep (\hat y_0^{(k)})$. This will form the basis of our general approach. 
We offer several methods of utilizing the available data, $(X_i,T_i,Y_i)$, to sample from $\cP_{Y_i|T_i=t,X_i=x}$ and estimate \cref{eq:gen_estimation}.

\subsection{Matching Approach}
\label{sec:matching}

We first propose a matching approach, for which we prove consistency under \Cref{hyp:positivity,hyp:sutva,hyp:unconfoundedness} and an additional regularity assumption.
We use a $k-$nearest neighbor approach that yields non-parametric consistent estimators for $\CPTE(x)$.
Let $\cN_t\subset[n]$ be the set of patients with treatment assignment $T_i=t$. For $x\in\cX$ and $k\in[n]$, let the $k-$neighborhood of $x$ in arm $t\in\set{0,1}$ be denoted as $\cN_t(x,k)$, and defined as the minimizer of $\sum_{i\in \cI} \NRM{x-X_i}^2$ over all sets $\cI\subset\cN_t$ of size $k$.
Let, for any $x\in\cX$:
\begin{equation*}
    \hat q_W(x)\eqdef \frac{1}{k^2}\sum_{(i,j)\in\cN_1(x,k)\times \cN_0(x,k)}w(Y_i|Y_j)\,,
\end{equation*}
be the distributional $k-$NN estimator, satisfying the following:

\begin{theorem}\label{thm:knn}
	Assume that \Cref{hyp:positivity,hyp:sutva,hyp:unconfoundedness} hold.
	If the function $x,x'\mapsto \esp{w(Y_i(1)|Y_j(0))|X_i=x,X_j=x'}$ is continuous, for all $x\in\cX$, then in probability:
	\begin{equation*}
		\hat q_W(x)\longrightarrow q_W\eqdef \CPTE(x)\,,
	\end{equation*}
	as long as $k\to\infty$ and $\frac{k\log(n)}{n}\to 0$, where $q_W$ stands for the CPTE.
\end{theorem}
The proof is in \Cref{app:proof_knn}.
For $x\in\cX$ and $(i,j)\in\cN_1(x,k)\times \cN_0(x,k)$, $w(Y_i|Y_j)$ can be interpreted as an approximate sample of $\cP_{Y_i|X_i=x,T_i=1}\times\cP_{Y_j|X_j=x,T_j=0}$, which can be used to calculate $\hat q_W$ with \Cref{eq:gen_estimation}. 
We can also define $\hat q_L$ as $w_L(y|y')=w(y'|y)$ for the `anti-preference'. Like the `preference' function $q_W(x)$, $\hat q_L(x)\to q_L(x)$ with distributional K-NN: 
\begin{equation}\label{eq:q_L}
     q_L(x) \eqdef \esp{w(Y_i(0),Y_j(1))|X_i=X_j=x}\,.
\end{equation}
\subsection{Quantile Regression Approach}
\label{sec:quantile_regression}

Quantile regression methods \citep{Koenker1978} can be used to estimate both parametrically or non-parametrically the inverse CDF of $Y(1,x)$ and $Y(0,x)$. This allows us to draw samples from the estimated inverse CDF which can be used to estimate $\esp{w(Y_i(0),Y_j(1))|X_i=X_j=x}$.
We present in our experiments results with linear quantile regression \citep{seabold2010statsmodels, greene2003econometric} and random forest quantile regression \citep{meinshausen2006quantile}.
For any given quantile model we can use the general method described in \Cref{algo:sampling_based_estimation} (\Cref{app:distrib_algo}) to sample from the inverse CDF, giving us \textit{distributional} regression methods.
\Cref{algo:sampling_based_estimation} can also be used as a basis for multivariate outcomes if $Y$ is $d-$dimensional by using the Rosenblatt transform~\citep{rosenblatt1952remarks}.

\section{Preference-based Optimal Policies}
\label{sec:policy}

\subsection{Policy Value}
\label{sec:policy_value}

A policy is a function $\pi:\cX\to\set{0,1}$ that maps features to treatment assignment.
With real-valued outcomes, an optimal policy is a policy $\pi^\star$ that maximizes the policy value, $\esp{Y_i(\pi(X_i))}$, over all policies $\pi$ in a policy class $\Pi\subset\set{0,1}^\cX$.

\textit{How is a policy's preference to be measured?}
We define the \textbf{'Preference' Policy Value}, which quantifies how preferable the treatments assigned by the policy are to the alternative. For example a clinical trial may wish to estimate how preferable treating all units with treatment is to control. Let the \textbf{'Preference' Policy Value} write as, for any $\pi:\cX\to\set{0,1}$:
\begin{equation}\label{eq:value_policy}
    V(\pi)\eqdef \esp{\pi(X_i)q_W(X_i)+(1-\pi(X_i))q_L(X_i)}\,,
\end{equation}
where $q_W(x)=\CPTE(x)$ and $q_L(x)$ are as in \Cref{eq:q_L}.
The goal of policy learning is then to learn a policy $\pi$ that maximizes ``preference'' over a given policy class~$\Pi\subset\set{0,1}^\cX$ \citep{athey2021policy}:
\begin{equation}\label{eq:policy_learning}
    \pi^\star\in\argmax_{\pi\in\Pi} V(\pi)\,.
\end{equation} 
We next propose different strategies for estimating $\pi^\star$.

\subsection{Optimal Treatment Rule}
\label{sec:OTR}

An optimal treatment rule (OTR) is a policy that solves \Cref{eq:policy_learning}. For unconstrained policy learning (\textit{i.e.}, $\Pi=\set{0,1}^\cX$ is the set of all possible policies), $\pi^\star$ has a closed form expression:
\begin{equation}\label{eq:OTR}
    \forall x\in\cX\,,\quad \pi^\star(x)=\one_\set{\delta(x)>0}\,,
\end{equation}
where:
\begin{equation}\label{eq:delta_OTR}
    \forall x\in\cX\,,\quad \delta(x)\eqdef q_W(x)-q_L(x)\,.
\end{equation}
The optimal policy recommends a treatment if the `preference', $q_W(x)$ is higher than the `anti-preference', $q_L(x)$.
$\delta(x)$ can also be interpreted as a \textit{conditional net benefit} of the treatment, usually defined over the whole population for hierarchical outcomes \citep{Buyse2010}. 

\textbf{Direct Plug-In Estimation of the OTR.}
\label{sec:plug_in_OTR}
Our first policy learning approach in estimating the OTR $\pi^\star$ defined in \Cref{eq:OTR}, consists in directly estimating $\delta$ defined in \Cref{eq:delta_OTR}.
With $\hat q_W$ and $\hat q_L$ (non-parametric) estimators of $q_W,q_L$, let $\hat \delta(x)\eqdef \hat q_W(x)-\hat q_L(x)$.
The estimated OTR then writes as:
\begin{equation}\label{eq:OTR_plugin}
	\hat\pi_\mathrm{OTR}(x)\eqdef \one_\set{\hat\delta(x)>0}\,.
\end{equation}
Depending on the precision with which the nuisance parameters $q_W$ and $q_L$ are estimated, we obtain consistent estimators of $\pi^\star$.
In particular, if $\hat q_W,\hat q_L$ are built with $k-$NN, \Cref{thm:knn} guarantees that $\hat \pi$ is a consistent estimator of $\pi^\star$.

\subsection{Plug-in Policy Value Optimization}
\label{sec:plug_in_value}

Instead of directly plugging-in estimators of $\delta,q_W$ or $q_L$ in the closed-form expression of the OTR $\pi^\star$ defined in \Cref{eq:OTR}, one could consider a more restrictive policy class $\Pi$ to optimize over.
This would require estimating the value of a policy (defined in \Cref{eq:value_policy}). 
A plug-in approach to estimate the value of a policy $\pi$ would be to estimate nuisance parameters $\hat q_W,\hat q_L$, and to estimate $V(\pi)$ by $\hat V(\pi)$, defined as:
\begin{equation}\label{eq:estim_value_policy}
    \hat V(\pi)\eqdef \frac{1}{n}\sum_{i=1}^n \pi(X_i)\hat q_W(X_i) + (1\!-\!\pi(X_i))\hat q_L(X_i)\,.
\end{equation}
The plug-in optimization approach then considers the following approximate optimization problem:
\begin{equation}\label{eq:estim_optim_policy}
	\hat\pi_\mathrm{plug-in}\in\argmax_{\pi\in\Pi}\hat V(\pi)\,.
\end{equation}
The class $\Pi$ to consider can be as follows.\\
\textbf{Weighted classification.}
In our experiments, we parameterize a linear classifier as $f_\theta$ for $\theta\in\Theta$ and predict the sign of $\hat q_W(X_i)-\hat q_L(X_i)$ weighted by the magnitude of the difference, implicitly maximizing $\hat V(\pi)$ over linear classifiers.\\
\textbf{Policy trees.}
Policytree \citep{sverdrup2020policytree} can also be used, and directly implements \Cref{eq:estim_optim_policy}, for $\Pi$ the set of all binary decision trees, of a maximum depth $D$ (specified by the user). 


\section{The 1-Step Corrected Policy}
\label{sec:semi_param_efficient}
\label{sec:1-step_EIF}

Plug-in estimators for the nuisance functions in \Cref{eq:estim_value_policy} can cause large bias in $\hat V(\pi)$ if the models have large non-asymptotic biases (e.g. nearest-neighbors), if they overfit (e.g. random forests), or if they are misspecified (e.g. linear models).
Efficient Influence Functions (EIFs) can be used to correct for this plug-in bias \citep{hines2022demystifying} and to obtain estimators that satisfy, under mild regularity conditions, double robustness properties and semiparametric efficiency.
Here, efficiency is understood relative to a fully nonparametric statistical model, in the sense that EIF-based estimators achieve the lowest possible asymptotic variance among all regular, asymptotically linear estimators in this model.
The existence of an efficient influence function further requires that the preference-based policy value parameter be pathwise differentiable with respect to the observed data distribution. As in classical policy value estimation, this holds under standard smoothness conditions on the parameter mapping and regularity of the preference function $w$.
Positivity ensures that the resulting EIF is well-defined and square-integrable, and that EIF-based estimators are regular and attain the semiparametric efficiency bound in a fully nonparametric model. Therefore, we compute an EIF $\phi(\mathcal P)(o)$ for any distribution $\mathcal P$ and observation $o$ for preference policy learning.
Details on and intuition behind EIF correction are presented in \Cref{sec:EIF_correction}.

\begin{proposition}\label{prop:EIF}
	Let $\psi_\pi(\cP)=V(\pi)$ be our estimand. For $Y(t, x)$ as defined in \Cref{def:CATE}, let:
\begin{align*}
	&p_W(x,t,y) = t\esp{w(y|Y(0,x))}\\
    & + (1-t)\esp{w(Y(1,x)|y)}\,,
\end{align*}
and similarly $p_L(x,t,y)$ by replacing $w$ by $w_L(y|y')=w(y'|y)$.
Let $e(x)=\proba{T_i=1|X_i=x}$
and $q_W,q_L$ as previously defined.
Then the EIF of $\psi_\pi$ for a distribution $\cP$ and $o=(X,T,Y)$ is:
\begin{align*}
	&\Phi_\pi(\cP)(o)  = \\
    & -\psi_\pi(\cP) + \pi(X) q_W(X) + (1-\pi(X))q_L(X)\\
	& \quad + \frac{T}{ e(X)}\Big(\pi(X)(p_W(X,T,Y)-q_W(X))\\
    &\qquad + (1-\pi(X))(p_L(X,T,Y)-q_L(X))\Big)\\
	& \quad + \frac{1-T}{1- e(X)}\Big(\pi(X)(p_W(X,T,Y)-q_W(X)) \\
    &\qquad + (1-\pi(X))(p_L(X,T,Y)-q_L(X))\Big)\,.
\end{align*}
\end{proposition}
The proof is given in \Cref{app:proof_EIF}.
Let $\hat e$ be an estimator of the propensity score, $\hat q_W,\hat q_L$  and $\hat p_W,\hat p_L$ estimators of the nuisance parameters $q_W,q_L,p_W,p_L$. 
The 1-step policy value estimator is the first-order correction of $\hat V(\pi)$ (\Cref{eq:value_policy}), and writes as $\Psi_\mathrm{1-step}^\pi \eqdef \Psi_\mathrm{plug-in}^\pi + \frac{1}{n}\sum_{i=1}^n \Phi_\pi(\hat \cP_n)(O_i)$, with $\Phi_\pi$ as in \Cref{prop:EIF} and with nuisance parameters replaced by their estimators $\hat e,\hat q_W,\hat q_L$  and $\hat p_W,\hat p_L$ (\Cref{eq:value_1_step_full} in \Cref{app:1-step}).
The 1-step optimal writes as:
\begin{equation}\label{eq:1-step_policy}
	\hat\pi_\mathrm{1-step}\in\argmax_{\pi\in\Pi} \set{\hat V_\mathrm{1-step}(\pi) \eqdef \Psi_\mathrm{1-step}^\pi}\,.
\end{equation}
This policy inherits the properties of the 1-Step correction for policy values and is expected to perform better than plug-in OTR (\Cref{sec:plug_in_OTR}) and plug-in value optimization policy (\Cref{sec:plug_in_value}).
In practice, the policy $\hat\pi_\mathrm{1-step}$ is computed as in \Cref{sec:policy_value}, using cross-fitted policy trees or weighted classification \citep{athey2021policy, chernozhukov2018double}. However, policy value $\hat V_\pi$ is replaced by its debiased value $\psi^\pi_\mathrm{1-step}$.
In our experiments, the 1-Step correction consistently improves policy performance when combined with a plug-in value optimization policy. 

\section{Experiments}
\label{sec:exps}
\FloatBarrier
\begin{figure*}[t]
  \centering
 \subcaptionbox{Synthetic experiment. Policy Learning with PNS policy value, heterogeneous RCT without correlation.\label{fig:synthetic_policy_learning_heterogeneous_RCT_uncorr}}%
    {
\includegraphics[width=\textwidth, trim=0 14cm 3.3cm 8.7cm, clip]{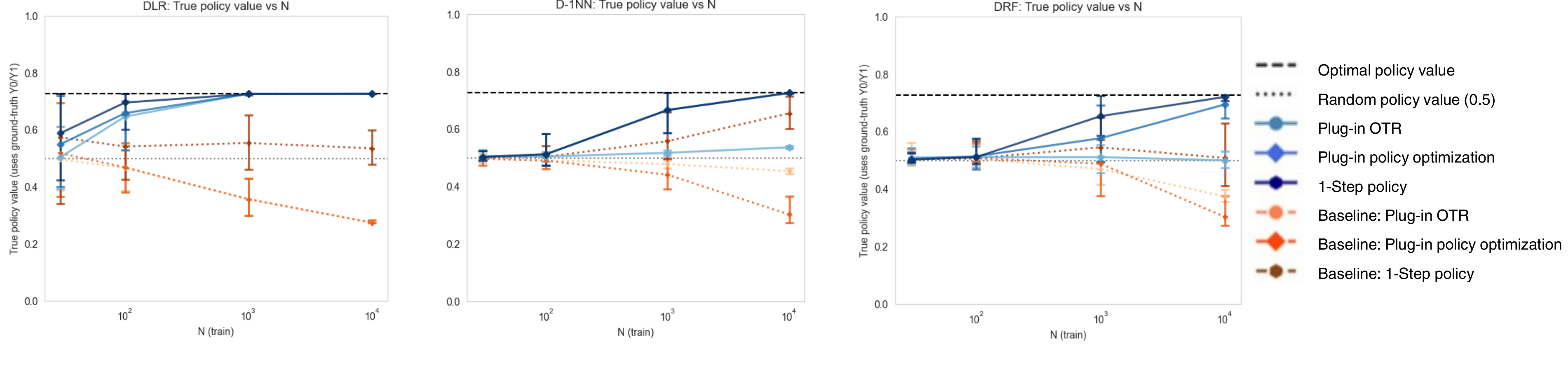}
}

\subcaptionbox{Semi-synthetic experiment. Policy learning with hierarchical outcomes, no potential outcomes correlation.\label{fig:semi_synth_policy_values_uncorr}}%
    {
\includegraphics[width=\textwidth, trim=0 16cm 3.3cm 8.7cm, clip]{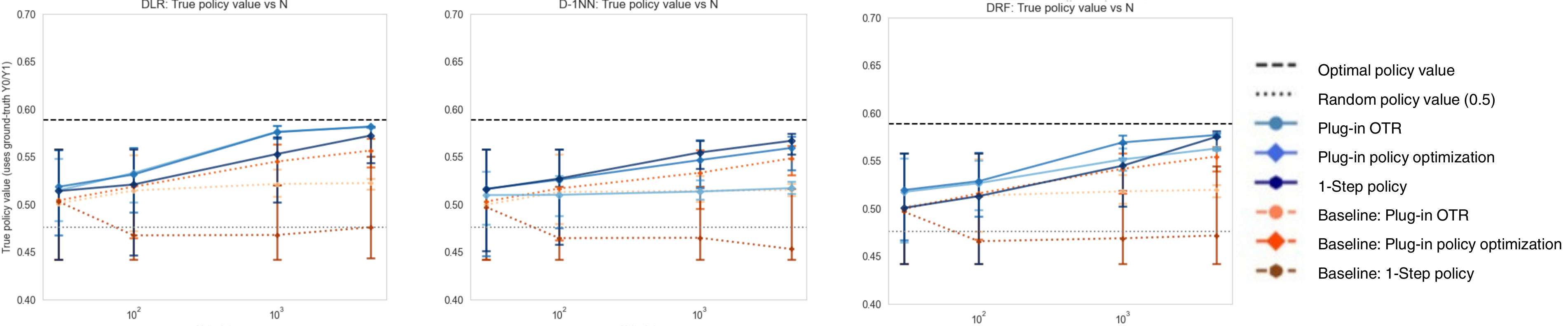}
}

  \caption{Policy learning results on synthetic and semi-synthetic experiments with no outcome correlations}
  \label{fig:policy_learning_main}
\end{figure*}

We conduct synthetic and semi-synthetic experiments to evaluate the performance of preference-based policy learning under our proposed approach, which estimates CPTE-related policy values. Our results demonstrate that the method consistently learns near-optimal policies even when key assumptions do not hold. We further benchmark its performance against both non-preference-based methods and policies with oracle access, highlighting the robustness and competitiveness of our approach.

\subsection{Synthetic experiments}

We consider a synthetic experimental setup with real-valued outcomes and $w(y|y')=\one_\set{y>y'}$, corresponding to the PNS application.
The data-generating process is detailed in \Cref{app:synth_exp_data_generating_process} and satisfies \Cref{hyp:positivity,hyp:sutva,hyp:unconfoundedness}.
Our synthetic experimental setup explores the benefits of preference-based policy learning with CPTE, contrasted with classical CATE policies. Outcomes are generated so that each unit's PNS prefers one treatment, while the CATE indicates the other one. This can occur when one of the potential outcomes distributions has a heavy tail distribution.
Figures~\ref{fig:large_div} and~\ref{fig:small_div} (\Cref{app:exps}) illustrate two examples of such potential outcome distributions. 
Our data-generating process compares the following properties and their impact on preference-based policy learning:
\textit{(i)} RCT vs. observational study;
\textit{(ii)} Homogeneous vs. heterogeneous treatment effect (\Cref{def:effect_modif});
\textit{(iii)} Correlation between counterfactual outcomes (negative, non-correlated, or positive correlation), violating the identifiability conditions of the Conditional ITE in \Cref{lem:treatment_effect_modifiers}.

We examine policy learning and evaluation for different combinations of the data generating parameters. We show results of the three different methods introduced in this paper to sample from the conditional outcome distribution $P(Y|X,T)$ and estimate the CPTE: distributional $k-$NN with varying values of $k$ as in \Cref{sec:matching} (with $k=1$ in \Cref{fig:policy_learning_main}), distributional random forests (DRF) and distributional linear regression (DLR) (\Cref{sec:quantile_regression}).
We compare our CPTE policies to CATE policies learned using a T-learner framework and corresponding base-learners: standard $k-$NN, ridge regression and regression forests \citep{kunzel2019metalearners}, which all estimate $\esp{Y|X,T}$ instead of sampling from $P(Y|X,T)$ like the distributional methods. 
These are denoted as \textit{baselines} in our figures.
The policy is then learned using the following approaches: plug-in OTR (\Cref{sec:plug_in_OTR}); plug-in policy value optimization (\Cref{sec:plug_in_value}); and policy value optimization with the EIF-corrected value (1-Step corrected policy, \Cref{sec:semi_param_efficient}).
Policies are trained 50 times, and evaluated on a held-out set with oracle knowledge.

In \Cref{fig:synthetic_policy_learning_heterogeneous_RCT_uncorr}, we present the heterogeneous RCT scenario with uncorrelated potential outcomes. The figure shows the preference policy value with 95\% bootstrap confidence intervals. The EIF correction consistently improves policy learning. Because the data-generating process is linear, DLR performs best, while DRF and KNN converge more slowly.
Results for policy value estimation are provided in \cref{fig:policy_eval_synth}, \Cref{app:synth_detailed_results}.
Plug-in and EIF policy evaluation converge to the true policy value, while the baseline CATE-based policies diverge to 0 or 1.
Policy learning results for other combinations of the data generating process (RCT/Observational, homogeneous/heterogeneous, correlated potential outcomes) are deferred to \Cref{fig:synthetic_policy_learning_grid}, \Cref{app:synth_detailed_results}.
Importantly, while correlated potential outcomes significantly bias policy value estimation, the learned policy itself still obtains good oracle performances. Even when the identifiability assumptions of \Cref{lem:treatment_effect_modifiers} do not hold, our methodology still learns good preference-based policies.

\subsection{Semi-synthetic experiments}

We conduct a semi-synthetic experiment inspired by the Tennessee STAR study \citep{star1990state}, a landmark randomized controlled trial on the effects of class size on student achievement. The study followed over 11,000 students from kindergarten through third grade across 79 schools in Tennessee, randomly assigning them to small, regular-sized or regular classes with a teacher’s aide. Student outcomes were measured in core subjects, including math and reading, along with detailed records of participation and attrition. In our setup, we treat math scores as a secondary continuous outcome and construct a binary primary outcome from student attrition, using the same preference function as in \Cref{eq:win_lexico} and \citep{Pocock2011winratio}. 
We collapse regular classrooms with and without aides, yielding 1,762 small-class students and 4,109 regular-class students after filtering.
We generate synthetic outcomes for this real data by training an MLP and using MC-dropout \citep{gal2015dropout}, approximating conditional distributions. These then serve as ground-truth for optimal policy and value. 
The full data generating process is detailed in \Cref{app:semi_synth_data}.

We compare the same estimation and policy learning methods as in the synthetic experiments.
The semi-synthetic setting allows us to control correlations between potential outcomes, while still being constrained by real-world covariates and approximating observed outcome distributions. 
We provide experiments on non-correlated potential outcomes, negative correlations and positive correlations, using the Iman-Conover method \citep{iman1982distribution}.
Results on the semi-synthetic STAR data show that `preference'-based policies outperform mean-based ones, as they balance hierarchical outcomes, prioritizing attrition while leveraging math scores. Unlike our synthetic experiments, the preference functions are close to the decision boundary, making policy learning more difficult. This is highlighted by
\Cref{fig:semi_synth_policy_values_uncorr}. DLR is misspecified, EIF corrects DRF and KNN for larger samples. Baseline methods perform poorly.
Similarly to the synthetic experiments, adding strong correlations between potential outcomes does not impact policy learning performance. \Cref{fig:semi_synthetic_policy_learning_grid,fig:semi_synth_policy_agreements_hierarchical} (\Cref{app:semi_synthetic}) respectively summarizes all semi-synthetic policy learning results and shows policy agreement results relative to the hierarchical oracle policy, and oracle policies that solely optimize the primary or secondary outcome.

  


  


\section{Conclusion}

This paper provides a general methodology for policy learning and for quantifying conditional treatment effects with complex outcomes that are best assessed via comparison.
We define
\textit{(a)} an unidentifiable conditional treatment effect (\Cref{eq:unidentifiable_CPTE});
\textit{(b)} an identifiable one (\Cref{def:CATE});
\textit{(c)} and offer estimators for the identifiable quantity (\Cref{sec:CPTE_estimation}).
Our experiments show that our estimators and policy learning methods efficiently learn preference-based policies, linking \textit{(b)} and \textit{(c)}. They also showcase when preference-based policies may be preferable to CATE-based policies. Lastly, we explore the effects of violating the identifiability assumptions connecting \textit{b} to \textit{a} by adding correlation to the potential outcomes in both the synthetic (\cref{app:synth_correlation}) and semi-synthetic (\cref{app:semi_synth_correlation}) experiments. We find policy evaluation to be sensitive to this violation, while policy learning was negligibly impacted.\\
Our work paves the way to several orthogonal research questions for future works, including performing sensitivity analyses to assess when the conditional measure defined in \Cref{def:CATE} is close to the non-identifiable one, and leveraging our framework with unknown preference functions $w$ that would need to be learned while optimizing policies.

\subsubsection*{Acknowledgments}
We thank the anonymous reviewers for their helpful feedback which helped improve the quality of our paper.
M. Even acknowledges funding from \href{https://www.theremia.health/home}{Theremia}. D.Parnas and U.Shalit were supported by ISF grant 2456/23.


\bibliographystyle{plainnat}

\clearpage

\clearpage
\appendix
\thispagestyle{empty}

\onecolumn

\section{Further specific details}

\subsection{Illustrative Adversarial Examples}
\label{app:examples}

\textbf{Example 1: observed and structural treatment effect modifiers.}
Let us place ourselves in a simple setting where we assume that features $X_i\in\cX$ \textit{do not} include treatment effect modifiers, denoted as $O_i\in\cO$.
	Assume that $\cY\subset\R$, $w(y|y')=\one_\set{y>y'}$, the set of treatment effect modifiers $\cO$ is of the form $\cO=\cO_1\cup\cO_2$ with $\proba{O_i\in\cO_1}=1-\alpha$, $\proba{O_i\in\cO_2}=\alpha$ and $O_i$ is independent of $X_i$, such that:
	\begin{align*}
		&Y_i(1)=y_1>y_0=Y_i(0)|O_i\in\cO_1\,,\\
        &Y_i(1)=y'_1<y'_0=Y_i(0)|O_i\in\cO_2\,,
	\end{align*}
	almost surely.
	We then have, if $y_0'>y_1>y_0>y_1'$:
	\begin{align}
		&\esp{w(Y(1,X_i)|Y_i(0))|X_i}=(1-\alpha)^2\,,\label{eq:none}\\
		&\esp{w(Y_i(1)|Y_i(0))|X_i}=1-\alpha\,,\label{eq:structural}\\
		&\esp{w(Y(1,X_i)|Y_i(0))|X_i,O_i}=\one_\set{O_i\in\cO_1}\label{eq:everything}\,.
	\end{align}
\Cref{eq:none} is the CPTE when we observe only $X_i$, which is non-informative. If we observe also the informative variable $O_i$, in \Cref{eq:everything}, heterogeneity is fully captured with the treatment effect modifier $O_i$.
However, \Cref{eq:structural} can also be written as the average of \Cref{eq:everything} conditionally on $X_i$: but we don't obtain the same value as in \Cref{eq:none}. This is due to the fact that $O_i$ is also a \textit{structural} treatment effect modifier.

\textbf{Example 2: linear CATE \textit{vs} CPTE.}
We now present an example showing different treatment recommendations between a classical linear CATE and our preference-based conditional effects. The example illustrates how CPTE leverages more than first-order moment information, and follows a majority-vote approach for treatment recommendations.
Let $\cY=\R$ and
\begin{align*}
    &Y_i(1) \sim \frac{\delta_{0}+\delta_{0.1} +\delta_{1.1}}{3}|X_i=x\,,\\
    &Y_i(0) \sim \frac{\delta_{-0.1}+\delta_{1}}{2}|X_i=x\,.
\end{align*}
In this case, $\CPTE(x)= 0.66>0.5$, while $\esp{Y_i(1)-Y_i(0)|x}=0.4-0.45=-0.05<0$.
Treatment recommendations differ depending on the policy maker's objective - whether they consider the magnitude of the effects (risk difference) or the most likely treatment to benefit unit $X_i$.
This archetype is the basis of our synthetic experiments.

\subsection{Translation to Structural Causal Model (SCM) terminology.}
\label{app:SCM}
Here we reformulate several of our basic definitions and theorem's within Pearl’s Structural Causal Model (SCM) framework. 
Let $(\mathcal{G}, U, f)$ be an SCM, where $U$ denotes exogenous variables, 
$V = (T, Y, X)$ endogenous variables, and
\[
Y = f_Y(T, X, U_Y).
\]
Counterfactual outcomes are defined by interventions as
\[
Y(t) := Y \mid \mathrm{do}(T=t).
\]

\paragraph{Preference-based individual causal effect.}
For a unit characterized by exogenous variables $U$, the individual-level
preference-based causal effect is
\[
\Delta(U_i) := w\big(f_Y(1, X_i, U_i),\, f_Y(0, X, U_i)\big)\,.
\]
Note that we have indeed $\Delta_i=\Delta(U_i)$.
The conditional ITE is then equal to:
\begin{equation*}
    \esp{w\big(f_Y(1, X_i, U_i),\, f_Y(0, X, U_i)\big)}\,.
\end{equation*}

\paragraph{SCM interpretation of CPTE - Definition \ref{def:CATE}.}
For a given covariate value $X=x$, the Conditional Preference Treatment Effect
$\mathrm{CPTE}(x)$
corresponds to comparing outcomes generated under independent draws of the
exogenous variables:
\[
U \sim P(U \mid X=x), 
\qquad
U' \sim P(U \mid X=x),
\qquad
U \perp U',
\]
and $Y(1,x)$, $Y(0,x)$ defined in \Cref{def:CATE} can also be defined as:
\[
Y(1,x) = f_Y(1, x, U_Y),
\qquad
Y(0,x) = f_Y(0, x, U'_Y),
\]
yielding the same CPTE definition:
\[
\mathrm{CPTE}(x)
=
\mathbb{E}\!\left[
w\!\big(Y(1,x), Y_i(0)\big)
\right] = 
\mathbb{E}\!\left[
w\!\big(Y_i(1), Y(0,x)\big)
\right]
\]

\paragraph{Structural treatment effect modifiers (SCM) - Definition \ref{def:structural_treatment_effect_modifiers}.}
A set of variables $Z_I \subseteq U$ is a set of \emph{structural treatment
effect modifiers} if for all $i,j\in[n]$ and all $x\in\R^\cI$:
\begin{equation*}
    \esp{w(Y_i(1),Y_j(0))\,|\, Z_i,Z_j,(Z_i^\cI,Z_j^\cI)=(x,x)}=g(x)\,,
\end{equation*}
where $z^\cI=(z_k)_{k\in\cI}$ for $z\in\R^p$.

\textbf{Conditions for CPTE equivalence to individual-level estimation - Theorem \ref{lem:treatment_effect_modifiers}.}\\
Assume that structural treatment effect modifiers in the sense of \Cref{def:structural_treatment_effect_modifiers} are included in the set of features $X_i$. Then, $\CPTE(x)=\esp{w(Y(1, U_i)|Y(0, U_i))|X_i=x}$ for all $x\in\cX$.

\subsection{Correcting the plug-in bias with efficient influence functions}
\label{sec:EIF_correction}

Plug-in estimators for the nuisance functions in \Cref{eq:estim_value_policy} can cause large bias to $\hat V(\pi)$ if the models have large non-asymptotic biases (e.g. nearest-neighbors), if they overfit (e.g. random forests) or if they are misspecified (e.g. linear models).
Efficient Influence Functions (EIFs) can be used to correct for this plug-in bias \citep{hines2022demystifying} and to obtain efficient estimators that satisfy, under mild assumptions, double robustness properties, semi-parametric efficiency optimality and minimal asymptotic variance among certain classes of estimators.

Let $\Psi(\cP)$ be an estimand we wish to estimate, that is the function of an observed distribution $\cP$.
In our case, $\cP$ is the distribution of our $n$ observations $O_i=(X_i,T_i,Y_i)$, and the estimand is $\Psi_\pi(\cP)= V(\pi)$, so that:
\begin{equation*}
    \Psi_\pi(\cP)=\esp{\pi(X_i) q_W(X_i)+ (1-\pi(X_i)q_L(X_i))}\,.
\end{equation*} 
As we saw before, in order to estimate $\Psi_\pi(\cP)$, we first need (non-parametric) nuisance estimates (of $q_W$ and $q_L$) learned on the data.
However, small perturbations and errors in these nuisance estimations can lead to biases in our estimate of $\Psi_\pi$ if we simply plug-in the nuisance estimators like in \Cref{eq:value_policy}. This is the \textit{plug-in policy estimator} of $\Psi_\pi$, denoted as $\Psi^\pi_\mathrm{plug-in}$.

\textbf{Efficient Influence Function (EIF).}
The Efficient Influence Function $\phi_\pi(\cP)$ of $\Psi_\pi$ at $\cD$ for any observation $o$ can be written:
\begin{equation*}
	\phi(\cP)(o) \eqdef \lim_{t\to 0} \frac{\dd \psi(\cP_t)}{\dd t}\,,\quad \text{where}\quad \cP_t = (1-t)\cP + t \delta_o\,,
\end{equation*}
where $\delta_o$ is a Dirac at $o$.
It should be understood as a ‘gradient' of $\Psi_\pi$, computed at $\cP$ and evaluated on observation $o$.
\citet{hines2022demystifying} provides methodologies to compute these EIFs with examples on classical estimands such as the ATE, ATT or RR.
The 1-step estimator corrects first-order biases that come from wrongly estimating the nuisance parameters of the distribution, where $\Psi^\pi_\mathrm{plug-in}= \hat V(\pi)$ (\Cref{eq:estim_value_policy}) and $\Phi^\pi$ is its EIF:
\begin{equation}\label{eq:value_1_step}
	\Psi_\mathrm{1-step}^\pi \eqdef \Psi_\mathrm{plug-in}^\pi + \frac{1}{n}\sum_{i=1}^n \Phi^\pi(\hat \cP_n)(O_i)\,.
\end{equation}

\subsection{Distributional regression algorithm}
\label{app:distrib_algo}
The meta-algorithm \Cref{algo:sampling_based_estimation} outputs estimators $\hat q_W,\hat q_L:\cX\mapsto[0,1]$ of $q_W(x)=\esp{w(Y_i(1),Y(0, X_i))|X_i=x}=\CPTE(x)$ and $q_L(x)=\esp{w(Y_i(0),Y(1,X_i))|X_i=x}$ (the CPTE with reversed treatments) from a given quantile models $M^{(t)}_q(x)$ and an interpolation rule.
First, \Cref{algo:sampling_based_estimation} discretizes $[0,1]$ into $\cQ$, and computes values $Y_q(t,x)$ for counterfactuals corresponding to these quantiles. This step estimates the inverse CDF.
Then $S$ other quantiles are sampled, and values corresponding to these quantiles are computed using the interpolation rule, sampling from the inverse CDF. \texttt{Interp} is any function that takes as inputs a set $(q,f(q))_{q\in\cQ}$ and a value $x\in[0,1]$ where $\cQ$ is a discretization of $[0,1]$, $f:[0,1]\to\cY$ is a model evaluated only on discretized values. \texttt{Interp} then outputs a value for $f(x)$.
The obtained values are then averaged in a final step.

\begin{algorithm}[H]
\caption{Sampling-based Estimation of $q_W, q_L$ via Quantile Models \label{algo:sampling_based_estimation}
}
\begin{algorithmic}[1]
\Require Quantile models $M_q^{(1)}(x)$ and $M_q^{(0)}(x)$; grid size $|\mathcal{Q}|$; number of samples $S$; interpolation rule \texttt{Interp}.
\Ensure Estimator $q_W, q_L$.

\State Sample a grid $\mathcal{Q}=\{q_1,\dots,q_{|\mathcal{Q}|}\}$ with $q_k \sim \mathrm{Unif}(0,1)$.
\State For each $q\in\mathcal{Q}$, compute
  \[
    Y_q(1,x) \gets M_q^{(1)}(x), \qquad
    Y_q(0,x) \gets M_q^{(0)}(x).
  \]
\State Draw $S$ independent quantiles $u_s^{(1)},u_s^{(0)} \sim \mathrm{Unif}(0,1)$ for $s=1,\dots,S$.
\State For each $s$, interpolate:
  \[
    Y_s(1,x) \gets \texttt{Interp}\!\big(\{(q, Y_q(1,x))\}_{q\in\mathcal{Q}},\, u_s^{(1)}\big),
  \]
  \[
    Y_s(0,x) \gets \texttt{Interp}\!\big(\{(q, Y_q(0,x))\}_{q\in\mathcal{Q}},\, u_s^{(0)}\big).
  \]
\State Estimate
  \[
    \widehat{q}_W(x) \gets \frac{1}{S}\sum_{s=1}^S
      w\!\big(Y_s(1,x),\,Y_s(0,x)\big).
  \]
  \[
    \widehat{q}_L(x) \gets \frac{1}{S}\sum_{s=1}^S
      w\!\big(Y_s(0,x),\,Y_s(1,x)\big).
  \]

\end{algorithmic}
\end{algorithm}

\subsection{The 1-Step policy value estimator}
\label{app:1-step}
The 1-step policy value estimator, defined in \Cref{eq:value_1_step}, writes as:
\begin{equation}\label{eq:value_1_step_full}
    \begin{aligned}
	\Psi_\mathrm{1-step}^\pi & = \frac{1}{n}\sum_{i=1}^n \Big[ \pi(X_i)\hat q_W(X_i)  + (1-\pi(X_i))\hat q_L(X_i)\\
	& \quad + \frac{T_i}{e(X_i)}\Big(\pi(X_i)(\hat p_W(X_i,T_i,Y_i)-\hat q_W(X_i))\\
    &\quad+ (1-\pi(X_i))(\hat p_L(X_i,T_i,Y_i)-\hat q_L(X_i))\Big)\\
	& \quad + \frac{1-T_i}{1- e(X_i)}\Big(\pi(X_i)(\hat p_W(X_i,T_i,Y_i)-\hat q_W(X_i)) \\
    &\quad+ (1-\pi(X_i))(\hat p_L(X_i,T_i,Y_i)-\hat q_L(X_i))\Big) \Big]\,.
\end{aligned}
\end{equation}

\section{Miscellaneous proofs}

\subsection{Proof of \Cref{lem:unidentifiable}}
\label{app:proof_non_ident}

\begin{proof}
    Assume that $\cY=\set{0,1}$, such that $Y_i(1),Y_i(0)$ are Bernoulli variables of mean $0.5$.
    We take $Y_i(1),Y_i(0)$ independent of $X_i$.
    Coupling $(Y_i(1),Y_i(0))|X_i=x$ as $Y_i(1)=Y_i(0)$ gives $\proba{Y_i(1)>Y_i(0)}=0$, while taking independent potential outcomes leads to $\proba{Y_i(1)>Y_i(0)}=\proba{Y_i(1)=1,Y_i(0)=0}=\frac{1}{4}$.
	Since the distribution of the observations $(X_i,T_i,Y_i)$ does not change by taking either coupling but the value of $\esp{w(Y_i(1)|Y_i(0))}$ changes, we have non-identifiability: $\esp{w(Y_i(1)|Y_i(0))}$ cannot be expressed as a function of the law of the observations.
\end{proof}

\subsection{Proof of \Cref{lem:treatment_effect_modifiers}}

\begin{proof}
    If $X_i=Z_i^\cI$ with $\cI$ the set of structural treatment effect modifiers or if $Z_i\subset X_i$, there exists $g$ such that $\esp{w(Y_i(1)|Y_j(0))|Z_i,Z_j,(X_i,X_j)=x}=g(x)$.
    Then,
    \begin{align*}
        \esp{w(Y_i(1)|Y_j(0)|X_i=X_j=x} &= \esp{\esp{w(Y_i(1)|Y_j(0)|Z_i,Z_j,(Z_i^\cI,Z_j^\cI)=(x,x)}|X_i=X_j=x}\\
        &= \esp{g(x)|X_i=X_j=x}\\
        &= g(x)\,.
    \end{align*}
    Now, taking $i=j$,
    \begin{align*}
        \esp{w(Y_i(1)|Y_i(0)|X_i=x} &= \esp{\esp{w(Y_i(1)|Y_i(0)|Z_i,Z_i^\cI=x}|X_i=x}\\
        &= \esp{g(x)|X_i=x}\\
        &= g(x)\,,
    \end{align*}
    concluding the proof.
    \end{proof}

\section{Proof of \Cref{thm:knn}}
\label{app:proof_knn}

Let $\mathbf{X_t}=\set{X_i, T_i=t}$ for arms $t=0,1$ and $\mathbf{X}=\mathbf{X}_0\cup\mathbf{X}_1$.
For any $x\in \cX$, let $Y(1,x)$ and $ Y(0,x)$ be independent copies of $Y_i(1)|X_i=x$ and $Y_i(0)|X_i=x$ respectively, which are also independent of each other. $Y(1,x)$ and $Y(0,x)$ are independent.\\
Let the $k-$nearest neighborhood of arm $t\in\set{0,1}$ of some point $x\in\cX$ be:
\begin{equation*}
    \cN_t(x,k) = \{a\in \mathbf{X_t} \ s.t. \ |\{b\in \mathbf{X_t}: ||b-x||<||a-z||\}| < k \}
\end{equation*}
A generic neighborhood, independent of treatment assignment, is designated $\cN(x,k)$.
First,
\begin{align*}
    \esp{\hat q(x)} & = \esp{w(Y_i|Y_j)|(i,j)\in\cN_1(x,k)\times \cN_0(x,k)}\\
    & = \esp{w(Y_i(1)|Y_j(0))|(i,j)\in\cN_1(x,k)\times \cN_0(x,k)}\quad\text{(\Cref{hyp:sutva,hyp:unconfoundedness})}\,.
\end{align*}

\textbf{Step 1: $\esp{\hat q(x)}\to q(x)$.}
We have:
\begin{align*}
\esp{\hat q(x)}-q(x) &= \esp{w(Y_i(1)|Y_j(0))|(i,j)\in\cN_1(x,k)\times \cN_0(x,k)} - \esp{w(Y(1,x), Y(0,x))} \\
    & = \esp{w(Y_i(1)|Y_j(0))-w(Y(1,x), Y(0,x))|(i,j)\in\cN_1(x,k)\times \cN_0(x,k)} \,.
\end{align*}
We then introduce the following classical lemma, which we prove for completeness.

\begin{lemma}\label{lem:neighbors_convergence_in_probability}
Let $x\in \cX$ and any $\hat{x} \in \cN(x,k)$. We have:
\begin{equation*}\hat{x} \overset{p}{\to} x\end{equation*}
when $k=o(n/log(n))$.
\end{lemma}

\begin{proof}
We need to show:
\begin{equation*}
\forall \epsilon>0\,,\quad   \lim_{n\to\infty} \mathop{P(||x-\hat{x}||>\epsilon})\,.
\end{equation*}
Let $\mathop{B}(x, \epsilon)$ be a sphere of $\epsilon$ radius around x and let $p = \P (\hat{x}\in \mathop{B}(x,\epsilon))$ be the probability of $\hat{x}$ being in the $\epsilon$-sphere.
We denote $Z = |\{z \in \mathbf{X}| z \in \mathop{B}(x, \epsilon)\}|$, the number of datum that fall within the $\epsilon$-sphere. 
We note that $Z \sim Bin(N, p)$, leading to:
\begin{equation*}
    \lim_{n\to\infty} \mathop{P(||x-\hat{x}||>\epsilon}) \leq \lim_{n\to\infty} \P (Z<k) = \lim_{n\to\infty} \sum_{i=0}^{k-1} \binom{n}{i}p^{i}(1-p)^{n-i}
\end{equation*}
If $p \in (0,1)$,
\begin{equation*}
    \lim_{n\to\infty} \sum_{i=0}^{k-1} \binom{n}{i}p^{i}(1-p)^{n-i} \leq \lim_{n\to\infty} k n^k q^n
\end{equation*}
where $q=\max(p,1-p)<1$ and where we bounded $\binom{n}{i}$ by $n^i\leq n^k$.
Thus, since $k n^k q^n=\exp(\log(k) + l\log(n) + n\log(q))$ and $\log(q)<0$, as long as $\frac{k\log(n)}{n}\to 0$:
\begin{equation*}
\lim_{n\to\infty} \mathop{P(||x-\hat{x}||>\epsilon})
\,,
\end{equation*}
hence the result.
Then, if $p\notin(0,1)$, either $p=0$ or $p=1$.
$p=0$ corresponds to $x\notin\mathrm{Supp(X_i)}$ which we do not consider (since $q$ will be evaluated on sampled datapoints) while if $p=1$ it means that $X_i$ is a deterministic random variable equal to $x$, meaning that the lemma is trivial.
\end{proof}

Since we assume that $w: (x, x') \rightarrow q(x, x')$ is continuous almost everywhere, we conclude using the continuous mapping theorem \citep{van2000asymptotic}.
Thus, $\esp{\hat q(x)}\to q(x)$.
To have consistency, note that
\begin{equation*}
    \P \left(|\hat{q}(x) - q(x)|>\epsilon\right) = \P \left(|\hat{q}(x) - \mathbb{E}[\hat{q}(x)] + \mathbb{E}[\hat{q}(x)] - q(x)|>\epsilon\right)\,.
\end{equation*}
Using an union bound,
\begin{equation*}
    \P \left(|\hat{q}(x) - q(x)|>\epsilon\right)\leq \P \left(|\hat{q}(x) - \mathbb{E}[\hat{q}(x)]| > \epsilon/2\right) + \P \left(|\mathbb{E}[\hat{q}(x)] - q(x)|>\epsilon/2\right)\,.
\end{equation*}
We thus need to show that:
\begin{equation*}
    \forall \eps>0\,,\quad \P \left(|\hat{q}(x) - \mathbb{E}[\hat{q}(x)]| > \epsilon\right)\to0\,.
\end{equation*}

\textbf{Step 2: $|\hat{q}(x) - \mathbb{E}[\hat{q}(x)]|\to 0$ in probability.} We have:
\begin{align*}
    \P \left(|\hat{q}(x) - \mathbb{E}[\hat{q}(x)]| > \epsilon\right) &= \proba{\left|\frac{1}{k^2} \sum_{(i,j)\in \cN_1(x,k) \times \cN_0(x,k)} w(Y^i, Y^j) - \mathbb{E}[w(Y^i, Y^j)]\right| > \epsilon}\\
    &\leq \frac{\mathbb{E}\left[\left(\sum_{(i,j)\in \cN_1(x,k) \times \cN_0(x,k)} w(Y^i, Y^j) -\mathbb{E}[w(Y^i,Y^j)]\right)^2\right]}{k^4\epsilon^2}\,,
\end{align*}
using Chebyshev's inequality.
We will denote $d(i,j) := w(Y^i,Y^j) - \mathbb{E}[w(Y^i,Y^j)]$.
We can rewrite the earlier expression as:
\begin{equation*}
    \frac{\mathbb{E}\left[\left(  \sum_{(i,j)\in \cN_1(x,k) \times \cN_0(x,k)} d(i,j)\right)^2\right]}{k^4\epsilon^2}\,.
\end{equation*}
We split the quadratic expressions into two groups, depending on if there is overlap between the pairs compared in the preference function. We will denote the four-way cross product by $\tilde{\cN} = \cN_1(x,k) \times  \cN_1(x,k) \times  \cN_0(x,k) \times  \cN_0(x,k)$. 
\begin{equation*}
    \frac{1}{k^4\epsilon^2} \sum_{(i, \hat{i}, j, \hat{j}) \in \tilde{N}} \left(\mathbb{I}\{i=\hat{i} \ or \ j=\hat{j}\} + \mathbb{I}\{i\neq\hat{i}\}\mathbb{I}\{j\neq\hat{j}\}\right)\mathbb{E}[d(i,j)d(\hat{i},\hat{j})]
\end{equation*}
\textit{Case 1}: There are $2k^3$ cases where at least $i=\hat{i}$ or $j=\hat{j}$. 
\begin{align*}
    & \mathbb{E}[d(i,j)d(\hat{i},\hat{j})] \\
    & = \cov[w(Y^i,Y^j),w(Y^{\hat{i}}, Y^{\hat{j}})]\qquad \text{(definition of covariance)} \\
    & \leq \var[w(Y^i, Y^j)]\\ 
    & \leq 1 \,,
\end{align*}
since $w$ is assumed to be bounded by 1.
Therefore, the summation may be bounded as follows:
\begin{align*}
    \frac{1}{k^4\epsilon^2} \sum_{(i, \hat{i}, j, \hat{j}) \in \tilde{N}} \mathbb{I}\{i=\hat{i} \ or \ j=\hat{j}\}\mathbb{E}[d(i,j), d(\hat{i}, \hat{j})] 
    &\leq  \frac{2k^3}{k^4\epsilon^2}\\
    & =  \frac{2}{k\epsilon^2}\,.
\end{align*}
\textit{Case 2}: There are $O(k^4)$ cases when $i\neq \hat{i}$ and $j\neq \hat{j}$. Because pairs are \textit{i.i.d.}, if there are no shared members the covariance is zero.
Therefore,
\begin{align*}
    \frac{1}{k^4} \sum_{(i, \hat{i}, j, \hat{j}) \in \tilde{N}} \mathbb{I}\{i\neq\hat{i}\}\mathbb{I}\{j\neq\hat{j}\}\mathbb{E}[d(i,j)d(\hat{i},\hat{j})] &= 
    \frac{1}{k^4} \sum_{(i, \hat{i}, j, \hat{j}) \in \tilde{N}} \mathbb{I}\{i\neq\hat{i}\}\mathbb{I}\{j\neq\hat{j}\}\mathop{Cov}[w(Y^i, Y^j), w(Y^{\hat{i}},Y^{\hat{j}})] \\ 
    &= 0 \,.
\end{align*}
Thus,
\begin{equation*}
    \P \left(|\hat{q}(x) - \mathbb{E}[\hat{q}(x)]| > \epsilon\right) \leq \frac{2}{k\epsilon^2}\,,
\end{equation*}
which converges to $0$ for any fixed $\eps>0$, as long as $k\to\infty$.
This concludes the proof.

\section{Proof of \Cref{prop:EIF}}
\label{app:proof_EIF}

\begin{proposition}\label{prop:EIF_app}
	Let $\psi_\pi(\cP)=V(\pi)$ be our estimand.
	\begin{enumerate}
		\item For $\pi(x)=1$ the constant policy, $\psi_\pi = \psi_W$ is the win proportion, i.e. the likelihood that the primary intervention is preferable to the alternative. Its EIF writes as; for $o=(X,T,Y)$:
		\begin{align*}
			\phi_W(\cP)(o) &= -\psi_W(\cP) + \esp{w(Y(1,X_i)|Y(0,X_i))|X_i=X}\\
			& \quad + \frac{T}{e(X)} \left( \esp{w(Y|Y(0,X_i))} - \esp{w(Y(1,X_i)|Y(0,X_i))|X_i=X} \right)\\
			& \quad + \frac{1-T}{1-e(X)} \left( \esp{w(Y(1, X_i)|Y)} - \esp{w(Y(1,X_i)|Y(0,X_i))|X_i=X} \right) \,,
		\end{align*}
		where $e(x)\eqdef \proba{T_i=1|X_i=x}$.
		\item For $\pi(x)=0$ the constant policy, $\psi_\pi = \psi_L$ is the loss proportion, i.e. the likelihood that the primary intervention is detrimental relative to the alternative, and its EIF writes as; for $o=(X,T,Y)$:
		\begin{align*}
			\phi_L(\cP)(o) &= -\psi_L(\cP) + \esp{w(Y(0, X_i)|Y(1,X_i))|X_i=X}\\
			& \quad + \frac{T}{e(X)} \left( \esp{w(Y(0,X_i)|Y)} - \esp{w(Y(0,X_i)|Y(1,X_i))|X_i=X} \right)\\
			& \quad + \frac{1-T}{1-e(X)} \left( \esp{w(Y|Y(1,X_i))} - \esp{w(Y(0,X_i)|Y(1,X_i))|X_i=X} \right) \,.
		\end{align*}
		\item For general policy $\pi$, the EIF of $\psi_\pi$ is:
		\begin{equation*}
			\phi_\pi(\cP)(o) = -\psi_\pi(\cP) + \pi(x)(\phi_W(\cP)(o) + \psi_W(\cP)) + (1-\pi(x))(\phi_L(\cP)(o) + \psi_L(\cP))
		\end{equation*}
	\end{enumerate}
\end{proposition}

\begin{proof}[Proof of \Cref{prop:EIF_app}]
	We have:
	\begin{align*}
		\Psi_W(\cP) & \eqdef \esp{\esp{w(Y_i(1,X_i)|Y_i)|X_i,T_i=0}}\\
		& = \int \dd\P_{X_i}(x) \dd \P_{Y_i|X_i=x,T_i=0}(y) \dd \P_{Y_i|X_i=x,T_i=1}(y')w(y'|y)\\
		& = \int p(x) p(y|x,0) p(y'|x,1) w(y'|y) \dd x \dd y \dd y'\\
		& = \int p(x) \frac{p(x,0,y)}{p(x,0)} \frac{p(x,1,y')}{p(x,1)} w(y'|y) \dd x \dd y \dd y'\,,
	\end{align*}
	where $p$ is the density of $\cP$ with respect to the Lebesgue measure if $\frac{\dd \cP}{\dd \lambda}<\infty$, and this expression naturally extends in the general case.
	Let $h\in(0,1)$ (that will be taken infinitesimally small) and $o = (X,T,Y)$.
	Let $\cP_h=(1-h)\cP_h\delta_{X,T,Y}$ be the mixture of these two distributions.
	We have:
	\begin{align*}
		\Psi_W(\cP_h) &= \int p_h(x) \frac{p_h(x,0,y)}{p_h(x,0)} \frac{p_h(x,1,y')}{p_h(x,1)} w(y'|y) \dd x \dd y \dd y'\,,
	\end{align*}
	where $p_h$ is the density of $\cP_h$. Note that we have $p_h = (1-h) p + h \delta_{X,T,Y}$, leading to:
	\begin{align*}
		\Psi_W(\cP_h) & = (1-h)\Psi_W(\cP) + o(h) + h \int p(y|X,0)p(y'|X,1)w(y'|y)\dd y\dd y' \\
		& \quad + h \frac{1-T}{1-\pi(X)} \int p(y'|X,1)w(y'|Y)\dd y' - h \frac{1-T}{1-\pi(X)} \int p(y|X,0)p(y'|X,1)w(y'|y)\dd y\dd y' \\ 
		& \quad + h \frac{T}{\pi(X)} \int p(y|X,0)w(Y|y)\dd y - h \frac{T}{\pi(X)} \int p(y|X,0)p(y'|X,1)w(y'|y)\dd y\dd y' \,.
	\end{align*} 
	Because $\Psi_W(\cP_h)=\Psi_W(\cP) + h \Phi(\cP)(o) +o(h)$, we can identify $\Phi_W$ as in \Cref{prop:EIF}.
	The proof for $\Phi_L$ and general policies $\Phi_\pi$ proceeds exactly in the same way.
\end{proof}

\section{Detailed Experimental Section}
\label{app:exps}

\subsection{Synthetic simulation details}
\label{app:synth_simulation_details}

We explore in our synthetic experiments the benefits of applying our CPTE methodology to preference-based policy learning and evaluation. We take as a use case PNS over a continuous outcome, $\esp{w(Y_i(1), Y_j(0))|X_i=X_j=x} = \proba{Y_i(1)>Y_j(0)|X_i=X_j=x}$.

In some cases a CATE-based policy suffices to optimize a CPTE target. For example, CATE estimation is sufficient to learn a binary PNS preference because the Bernoulli distribution is characterized by its expectation. However, complex potential outcome distributions can result in divergence between an optimal CATE policy and an optimal preference-based policies, such as PNS. We give a real-world example below.

To create our experimental data, we simulate a mix of binary and continuous features. The potential outcomes are generated from a combination of linear effect modifiers and a complex treatment effect. We present two scenarios; one with homogeneous treatment effects and the other with a simple heterogeneous treatment effect. The simulation maintains \Cref{hyp:positivity,hyp:sutva,hyp:unconfoundedness}.
To test our CPTE methodology against classical CATE-based policies, the complex treatment effect was simulated such that the two policies disagree on every unit. This can occur when one of the potential outcome distributions has a heavy tail. Real-world applications of this could be risky surgeries that benefit certain individuals greatly but are harmful for most patients. In this case, the minority of patients who benefit greatly may skew the average, making the riskier surgery appear more beneficial than a more conservative option whose likelihood of benefiting the patient is better. A CATE policy would recommend the risky surgery; a CPTE policy would not.\\
Our simulation mimics this scenario by drawing one potential outcome from a gaussian distribution and the other potential outcome from a bimodal normal distribution. Figures~\ref{fig:large_div} and~\ref{fig:small_div} illustrate two examples of potential outcome distributions for which there is divergence between a PNS-based policy and a CATE-based policy. Other more subtle examples can be simulated by using heavy-tail distributions such as Gamma and Pareto. 

\begin{figure}[H]
    \centering
    \begin{subfigure}[t]{0.4\textwidth}
        \centering
        \includegraphics[width=0.9\linewidth]{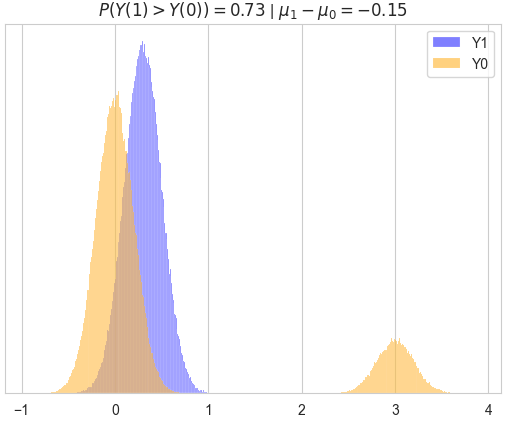}
        \caption{Large divergence}
        \label{fig:large_div}
    \end{subfigure}%
    \hfill
    \begin{subfigure}[t]{0.4\textwidth}
        \centering
        \includegraphics[width=0.9\linewidth]{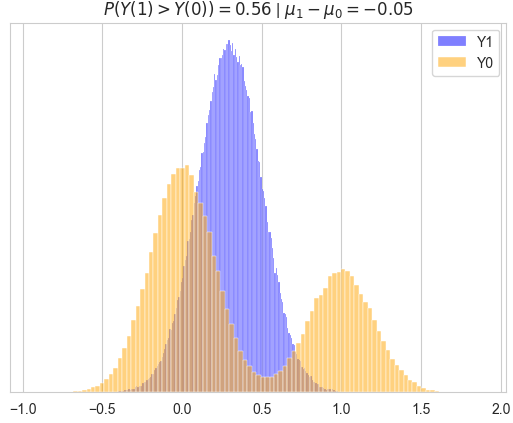}
        \caption{Small divergence}
        \label{fig:small_div}
    \end{subfigure}
    \caption{Potential outcome distributions}
    \label{fig:potential_outcome_distribution}
\end{figure}

In the homogeneous setting, all units potential outcomes are drawn from the same potential outcome distribution modified by a linear combination of the features. In the heterogeneous setting, an observed binary feature was used to decide which potential outcome was assigned the bimodal distribution. We include more details in \cref{app:synth_exp_data_generating_process}.

For each simulation we perform policy learning and evaluation with three distributional learning paradigms - KNN, linear regression and random forest. These are evaluated on a held-out set (N=10000) generated using a random seed that is not used for any of the training data simulations. Training data is simulated 50 times for each N. We sample a training set for each $N\in [30, 100, 1000, 10000]$ and train our distributional and baseline estimation methods. Using their estimates for the conditional PNS, we evaluate three policy rules. The first, OTR (\cref{sec:OTR}), uses the $\delta(x) > 0$ rule; $\delta(x)=q_W(x)-q_L(x)$ for CPTE and $\delta(x)=\esp{Y_i(1)-Y_i(0)|x}$ for CATE. The second policy rule, policy value maximization (\cref{sec:plug_in_value}) uses a logistic regression model to learn $f:X\rightarrow \mathbf{I}{(\delta(x) > 0)}$ over the training data. The policy follows the prediction. The third policy rule also uses a logistic regression model to decide policy, but learns over the EIF correction of $\delta(x)$ (\cref{eq:1-step_policy}).

We compare our CPTE approach using distributional estimation to a standard approach using CATE, which serves as our baseline. Distributional KNN (D-KNN), estimating $\cP_{Y_i|T_i,X_i}$ is compared to regular KNN, estimating $\esp{Y_i|T_i, X_i}$. Likewise, Distributional Linear Regression (DLR) using quantile regression is compared to regular ridge regression; Distributional Random Forest (DRF) is compared to regular random forest.

\subsubsection{Synthetic Data Generating Process}
\label{app:synth_exp_data_generating_process}
We generate features $X \in \mathbb{R}^{10}$ using 8 continuous variables and two binary variables. We randomly sample linear coefficients to create a shared baseline dependent on $X$ for the potential outcomes. In the observational setting, we choose three features (one discrete and two continuous) to act as confounders and use a linear transformation to create the logit for treatment arm selection. We ensure that overlap exists and set the expected treatment rate to 0.5.

\begin{itemize}
    \item Continuous features: $X_{cont}\sim \mathcal{N}(0,1)$
    \item Discrete features: $X_{disc}\sim Ber(0.5)$
    \item Homogeneous treatment assignment: $T\sim Ber(0.5)$
    \item Heterogeneous treatment coefficients: $\beta_T\sim U(0.2, 0.6)$ for one discrete and two continuous features acting as confounders, otherwise $0$.
    \item Heterogeneous treatment assignment: $T\sim Ber(X^T\beta_T)$
    \item Baseline outcome coefficients: $\beta_Y\sim U(0.1, 0.5)$ for all features. 
    \item Shared baseline Y: $Y_{baseline}=X^T\beta_Y$
\end{itemize}

We define two distributions for the potential outcome space: a bimodal mixture of Gaussians and a simple Gaussian distribution. The bimodal distribution has a higher expected value than the simple Gaussian, as visualized in \cref{fig:large_div} and \cref{fig:small_div}. However, the probability that a random variable of the simple Gaussian distribution is higher than the bimodal distribution is above 0.5. We let the $b\in[0,1]$ parameter control how many are part of the lower Gaussian and $1-b$ how many are part of the higher Gaussian. Referencing the risky surgery example mentioned above, $1-b$ represents how many patients are in the minority group that greatly benefited from the risky surgery.\\
Potential outcome parameters:

\begin{itemize}
    \item Shared variance: $\sigma = 0.2$
    \item Simple Gaussian parameters: $\mu_s=0.3$
    \item Simple Gaussian: $y_{simple} \sim \mathcal{N}(\mu_s, \sigma^2)$
    \item Mixture of Gaussians parameters: $\mu_{m,low}=0, \mu_{m, high}=3, b=0.85$
    \item Mixture of Gaussians: $y_{mixture} =  b*\mathcal{N}(\mu_{m,low}, \sigma^2) + (1-b)\mathcal{N}(\mu_{m,high}, \sigma^2)$ 
\end{itemize}

We create the potential outcomes by combining the linear baseline outcome with these distributions. In the homogeneous setting,
\begin{equation*}
    Y(1) = y_{baseline} + y_{simple}
\end{equation*}
\begin{equation*}
    Y(0) = y_{baseline} + y_{mixture}
\end{equation*}
Therefore, for all units the optimal PNS policy chooses T=1 and the optimal CATE policy chooses T=0.

We denote one of the binary features as $X_0$.
In the heterogeneous setting,
\begin{equation*}
    Y(1) = y_{baseline} + \mathbf{I}(X_0=1)y_{simple} + \mathbf{I}(X_0=0)y_{mixture}
\end{equation*}
\begin{equation*}
    Y(0) = y_{baseline} + \mathbf{I}(X_0=0)y_{simple} + \mathbf{I}(X_0=1)y_{mixture}
\end{equation*}
Therefore, the optimal PNS and CATE policies depend on the binary feature $X_0$.

\subsubsection{Experiment Details}
\label{app:synth_experiment_details}

\textbf{Evaluation.}
Evaluation was performed on a large hold-out set (N=50,000). $X$ and $y_{simple}, y_{mixture}$ were generated using a random seed not used by any of the training sets. The parameters for $\beta_y$ were generated using a random seed shared across all datasets. Therefore, $y_{baseline}$ comes from the same distribution for both the evaluation data and the training data.\\
\textbf{Training.}
For each $N\in\{30, 100, 1000, 10000\}$ the training data was simulated 50 times.  This allows us to observe the performance of the models over exponentially growing datasets. Each time $X$ and $y_{simple}, y_{mixture}$ were generated using a different random seed.\\
For each learning paradigm - KNN, linear regression and random forests - we train a distributional model and a baseline model. The distributional model estimates the distribution of the potential outcomes, $\hat{Y}_i(1), \hat{Y}_i(0)|X_i=x$  under \Cref{hyp:positivity,hyp:sutva,hyp:unconfoundedness}. Distributional KNN (DKNN) uses the nearest neighbors to approximate the distribution of the potential outcomes. We test nearest neighbor (k=1) and logarithmically growing k's. The distributional linear and forest regressors approximate the inverse cdf using quantiles and then sample 1000 times from the approximate potential outcome distribution. We denote the number of samples used as $S$. The baseline learner estimates the expected value of the potential outcome under \Cref{hyp:positivity,hyp:sutva,hyp:unconfoundedness}. These estimations allow us to perform policy learning and evaluation. \\
Policies are learned in three ways:
\begin{itemize}
    \item Plug-in of the OTR (\cref{sec:OTR})
    \item Plug-in Policy Value Optimization (\cref{sec:plug_in_value})
    \item EIF Policy Value Optimization (see \cref{sec:1-step_EIF})
\end{itemize}

\textbf{Policy Learning \& Policy evaluation}
For each learning paradigm, we compare the CPTE-policies to the CATE-policies using the oracle policy value on the evaluation set. This allows us to understand which policy perform best. The full results can be found in \cref{fig:synthetic_policy_learning_grid} and are analyzed in the next section. 

Policy evaluation was performed by using each paradigm's estimates on the evaluation set and comparing the optimal policy's value using them. The CATE policy value was calculated by plugging in the estimated expected values into the preference function (\cref{win_function_def}). 

\textbf{Model Parameters} 
The baseline linear regression model is Ridge with alpha=1.0.

The random forest parameters depend on the size of N. These were found by performing a hyperparameter search for each N using a random seed not used during training. This allowed us to understand reasonable parameters given different training set sizes.

\begin{table}[H]
\centering
\caption{Random Forest Regressor hyperparameter settings by sample size $N$.}
\label{tab:rfr_configs}
\begin{tabular}{lcccc}
\toprule
Condition on $N$ & Max Depth & Max Features & Min Samples Split & \# Trees ($n_{\text{estimators}}$) \\
\midrule
$N \geq 10000$            & 25 & 0.35 & 5 & 500 \\
$100 < N < 10000$      & 15 & 0.50 & 7 & 400 \\
$N \leq 100$          & 15 & 0.60 & 5 & 50 \\
\bottomrule
\end{tabular}
\end{table}

\subsubsection{Synthetic Experiments: Detailed Results}
\textbf{Policy learning.}
As we see in \Cref{fig:synthetic_policy_learning_grid}, the distributional approaches tend to converge to the optimal policy value. We also note that the variance, represented by bootstrapped 95\% confidence intervals, tends to get smaller for the distributional approaches. In contrast, the baseline learners tend to converge to the incorrect CATE policy and the variance tends to be larger. The simulation is linear; resulting in the quick convergence of all Distributed Linear Regression (DLR) policies to the optimal policy in all settings and for all policy learning methods. The baseline plug-in and direct policy value optimization converge to the optimal CATE policy, the opposite of the optimal CPTE policy. We note that the EIF PNS correction of the CATE facilitates above-random performance in some of the settings. The Distributed KNN (DKNN) policies also converge towards the optimal policy value. Despite being misspecified to the linear simulation, the DKNN plug-in policy correctly estimates the PNS, allowing the correct policy to be learned as a logarithmic rate as anticipated by \Cref{thm:knn}. Unlike DKNN, the Distributional Random Forest (DRF) method struggles to overcome its misspecification in the heterogeneous setting. Although not converging to the CATE policy like the baseline random forest, it struggles to learn the correct distribution. This results in a plug-in that stays close to a random policy because the heterogeneity is equally balanced. Despite this inadequacy, the policy value optimization methods converge to the optimal policy.

\FloatBarrier
\label{app:synth_detailed_results}
\begin{figure*}
  \centering
  \subcaptionbox{DLR\label{fig:dlr_rct_homo}}%
    {\includegraphics[width=0.30\textwidth]{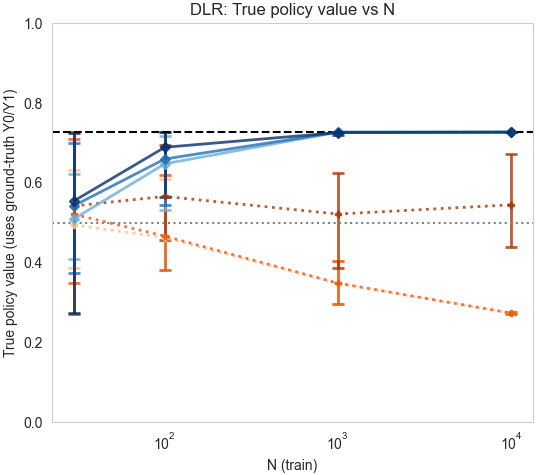}}
  \hfill
  \subcaptionbox{DKNN\label{fig:dknn_rct_homo}}%
    {\includegraphics[width=0.30\textwidth]{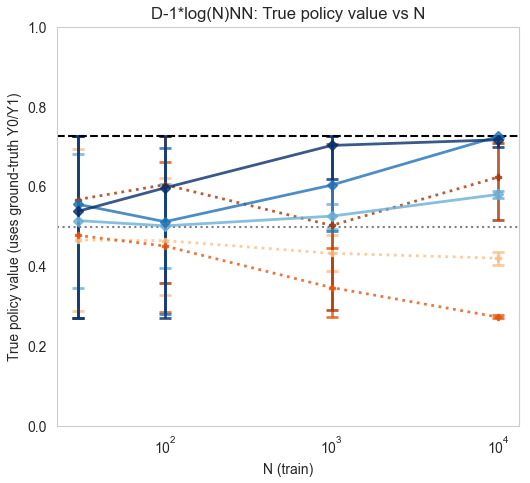}}
  \hfill
  \subcaptionbox{DRF\label{fig:drf_rct_homo}}%
    {\includegraphics[width=0.30\textwidth]{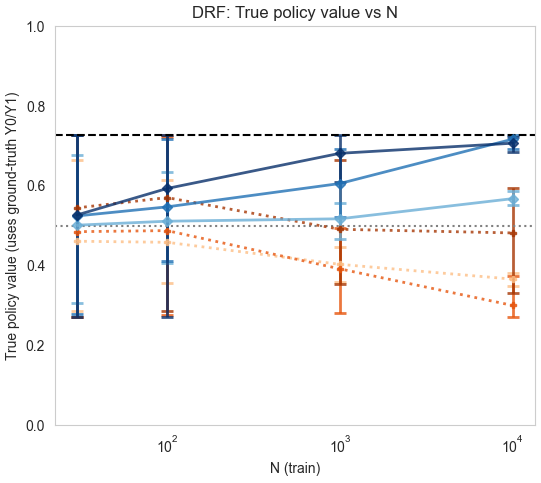}}

  \medskip
  \subcaptionbox{DLR\label{fig:dlr_rct_hetero}}%
    {\includegraphics[width=0.30\textwidth]{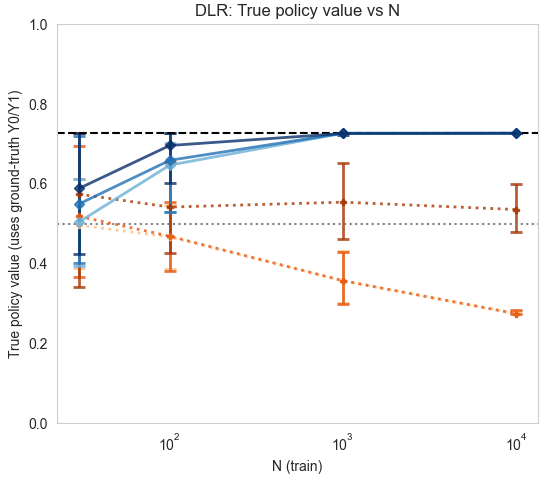}}
  \hfill
  \subcaptionbox{DKNN\label{fig:dknn_rct_hetero}}%
    {\includegraphics[width=0.30\textwidth]{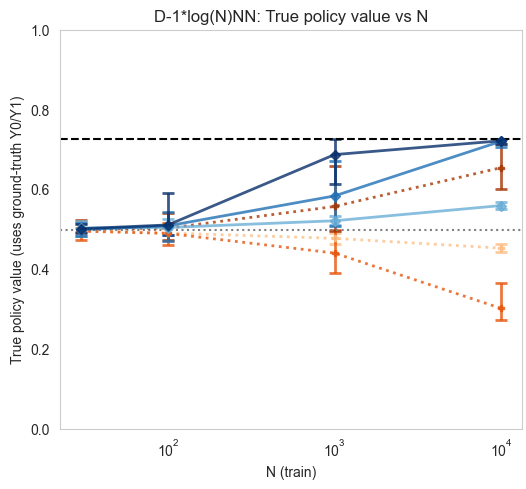}}
  \hfill
  \subcaptionbox{DRF\label{fig:drf_rct_hetero}}%
    {\includegraphics[width=0.30\textwidth]{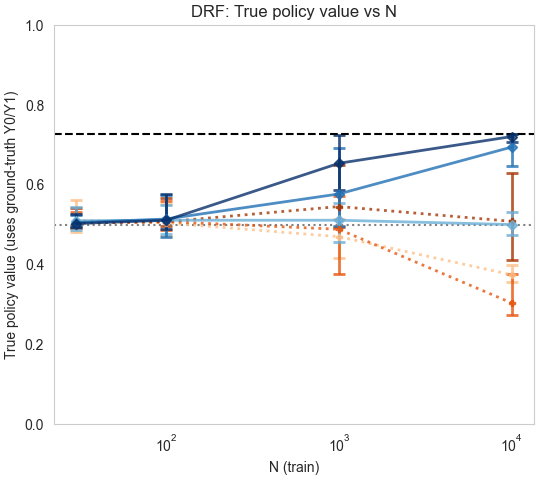}}

  \medskip
  \subcaptionbox{DLR\label{fig:dlr_obs_homo}}%
    {\includegraphics[width=0.30\textwidth]{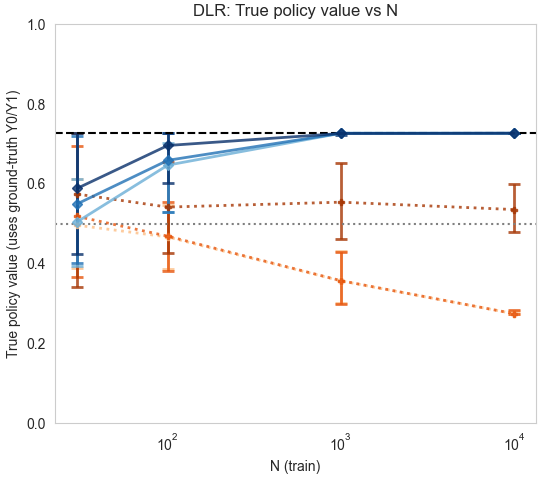}}
  \hfill
  \subcaptionbox{DKNN\label{fig:dknn_obs_homo}}%
    {\includegraphics[width=0.30\textwidth]{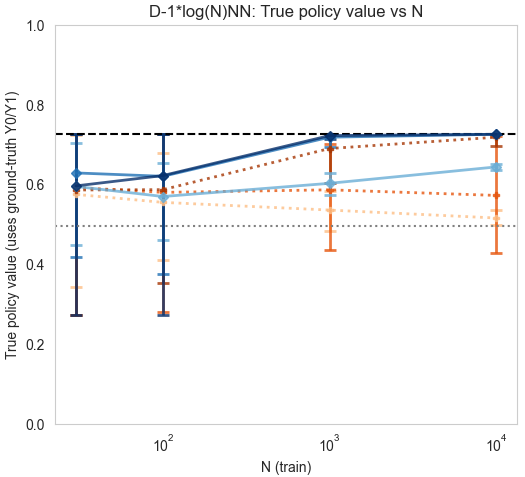}}
  \hfill
  \subcaptionbox{DRF\label{fig:drf_obs_homo}}%
    {\includegraphics[width=0.30\textwidth]{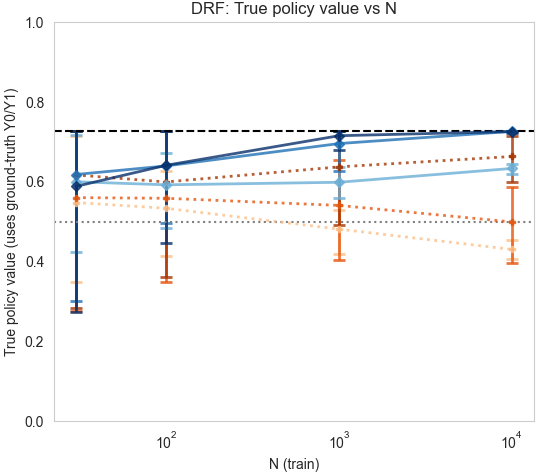}}
    
  \medskip
  \subcaptionbox{DLR\label{fig:dlr_obs_hetero}}%
    {\includegraphics[width=0.30\textwidth]{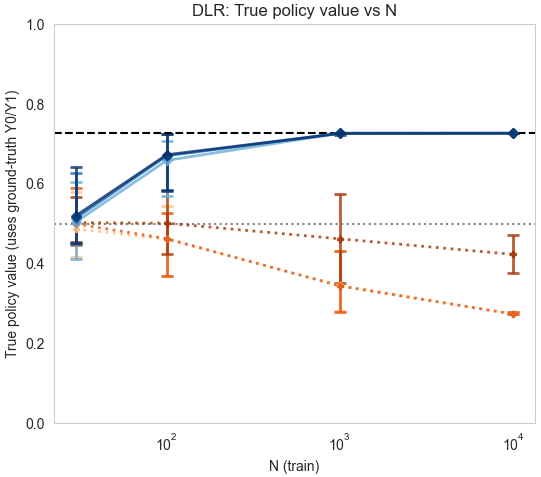}}
  \hfill
  \subcaptionbox{DKNN\label{fig:dknn_obs_hetero}}%
    {\includegraphics[width=0.30\textwidth]{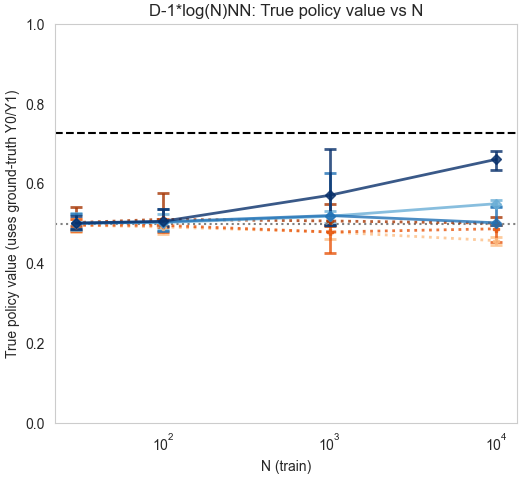}}
  \hfill
  \subcaptionbox{DRF\label{fig:drf_obs_hetero}}%
    {\includegraphics[width=0.30\textwidth]{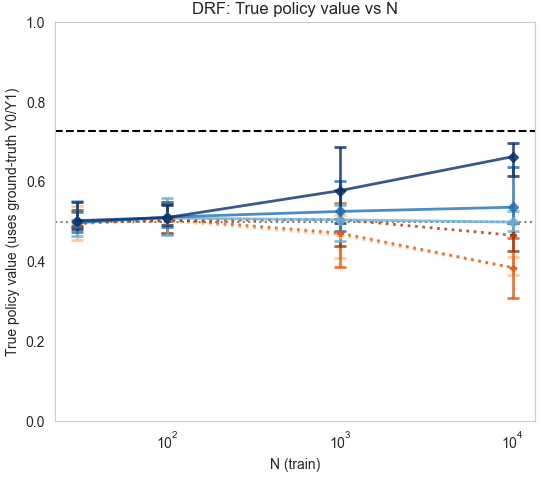}}

  \caption{Policy Learning with PNS policy value, no correlations. (a–c) Homogeneous RCT setting; (d–f) Heterogeneous RCT setting;
  (g-i) Homogeneous Observational setting; (j-l)Heterogeneous Observational setting.
  \textbf{Legends} as in \Cref{fig:synthetic_policy_learning_heterogeneous_RCT_uncorr}.
  \label{fig:synthetic_policy_learning_grid}}
\end{figure*}

\textbf{Policy evaluation.} 
Beyond the distributional method's ability to learn policies over complex potential outcome distributions, they also enable correctly-specified policy evaluation for complex preference functions like PNS. We evaluate the optimal policy in term of PNS using PNS estimates from the distributional and baseline models. As can be seen in \cref{fig:policy_eval_synth}, the distributional estimates converge towards the optimal policy. The correctly specified Distributional Linear Regression converges quickly to the correct policy value. DKNN and DRF tend towards the correct policy value. In contrast, the baseline estimation using the expected potential outcomes diverge to 0 and 1. We note that the true policy value is 0.72.

\begin{figure*}
  \centering
  \subcaptionbox{DLR\label{fig:policy_eval_dlr}}{%
    \includegraphics[width=0.9\linewidth]{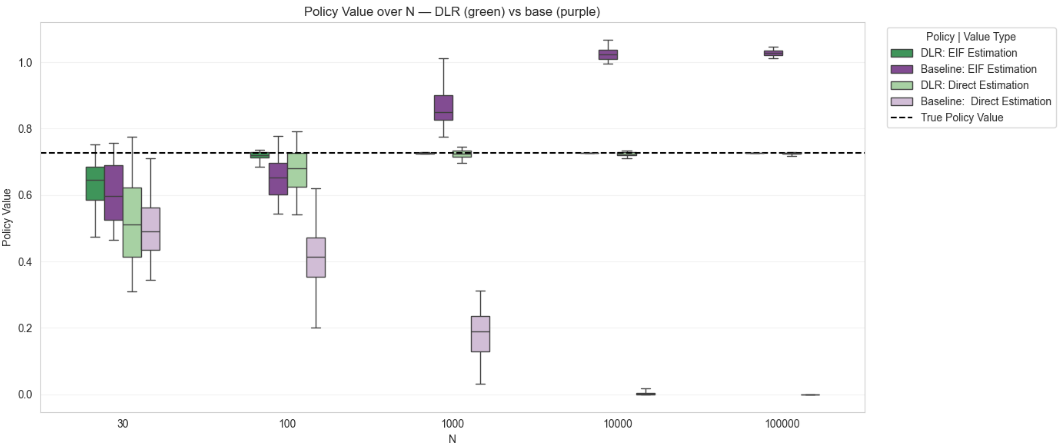}} \\[0.8em]
  \subcaptionbox{DKNN\label{fig:policy_eval_dknn}}{%
    \includegraphics[width=0.9\linewidth]{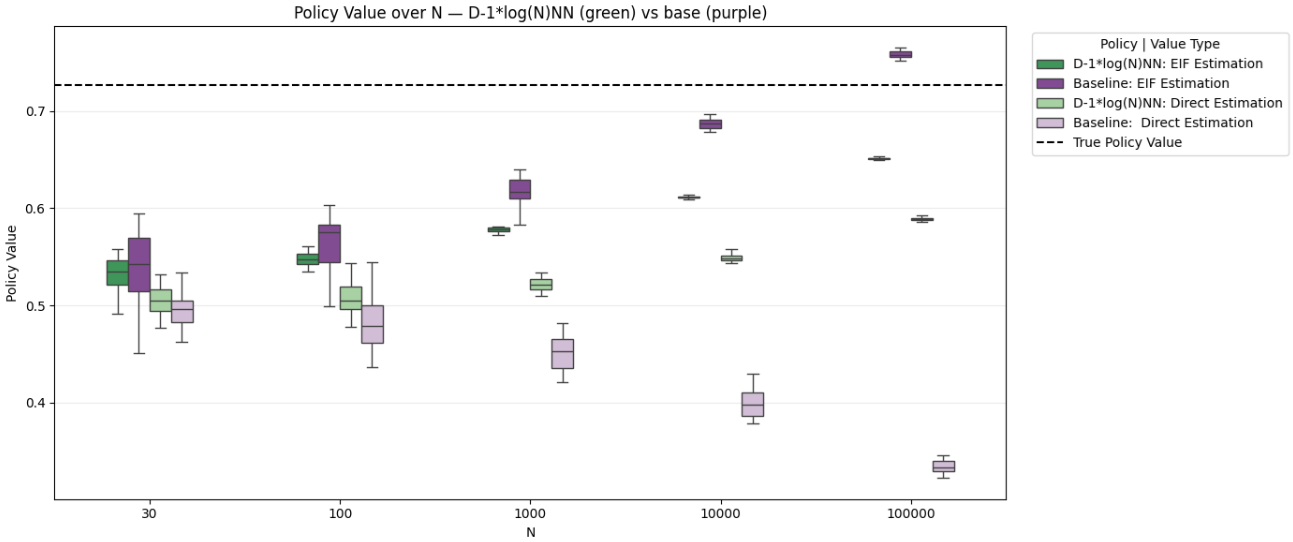}} \\[0.8em]
  \subcaptionbox{DRF\label{fig:policy_eval_drf}}{%
    \includegraphics[width=0.9\linewidth]{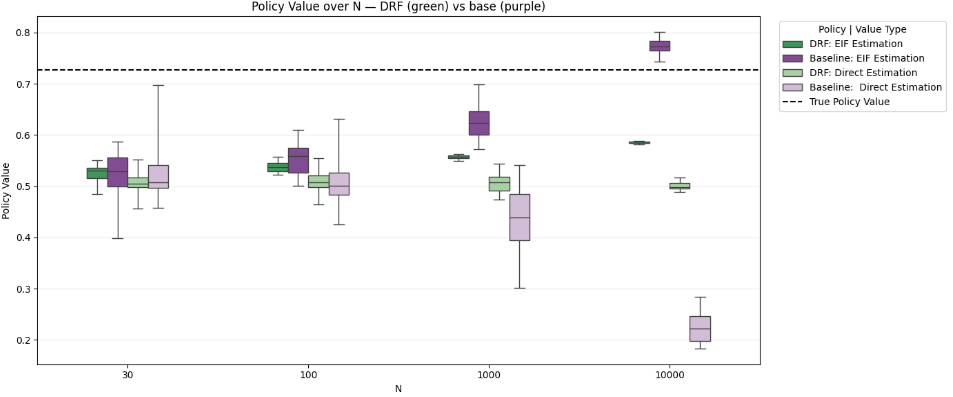}}

  \caption{Policy Evaluation of optimal policy in the synthetic, heterogeneous RCT setting}
  \label{fig:policy_eval_synth}
\end{figure*}

\textbf{Correlation between potential outcomes.} \label{app:synth_correlation}
We also examine when the assumption made in \Cref{lem:treatment_effect_modifiers}, i.e. all structural treatment effect modifiers must be observed, is violated. This assumption is unlikely to hold; these experiments can help us understand the impact of being far from this assumption. We test this by introducing correlations between the potential outcomes that are not explained by the available features.
Continuing the surgery example, patients for whom the risky surgery is successful may have a genetic component that also affects their reaction to the more conservative treatment.

We show that in our simple synthetic setting, policy learning appears to be more robust to the violation of this assumption than policy evaluation (\Cref{fig:policy_learning_correlation}). For the policy learning, we focus on the successful DKNN and distributional linear regression in the difficult heterogeneous setting. We add significant amount of positive and negative correlation between the potential outcomes with empirical Pearson correlations of 0.7 and -0.7 respectively.



\begin{figure*}
  \centering
  \subcaptionbox{DLR: -0.7\label{fig:dlr_hetero_neg07}}%
    {\includegraphics[width=0.30\textwidth]{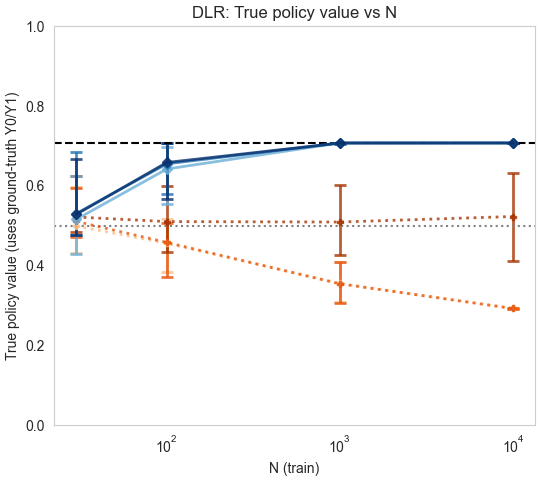}}
  \hfill
  \subcaptionbox{DLR: 0\label{fig:dlr_rct_hetero}}%
    {\includegraphics[width=0.30\textwidth]{dlr_rct_hetero.png}}
  \hfill
  \subcaptionbox{DLR: 0.7\label{fig:dlr_hetero_07}}%
    {\includegraphics[width=0.30\textwidth]{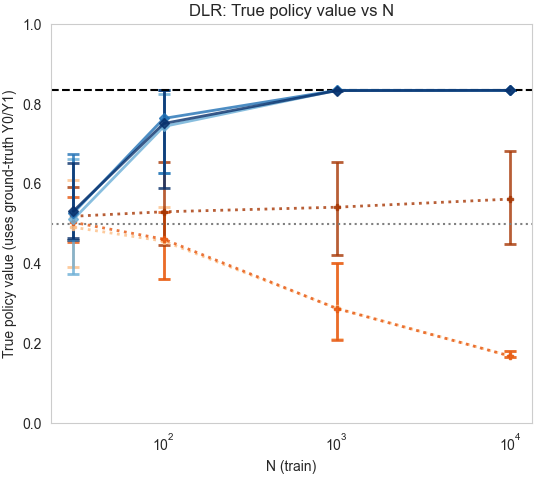}}

  \medskip
  \subcaptionbox{DKNN: -0.7\label{fig:dknn_hetero_neg07}}%
    {\includegraphics[width=0.30\textwidth]{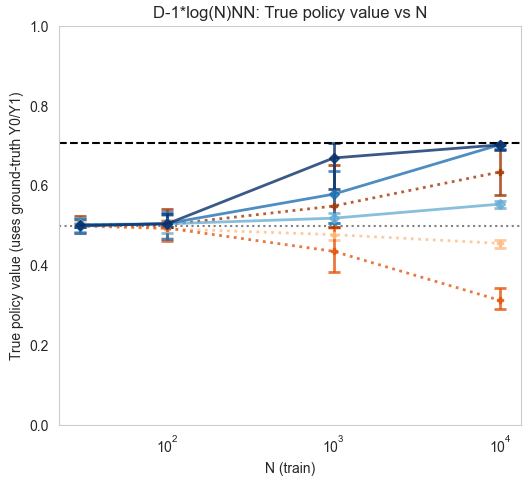}}
  \hfill
  \subcaptionbox{DKNN: 0\label{fig:dknn_rct_hetero}}%
    {\includegraphics[width=0.30\textwidth]{dknn_rct_hetero.png}}
  \hfill
  \subcaptionbox{DKNN: 0.7\label{fig:dknn_hetero_07}}%
    {\includegraphics[width=0.30\textwidth]{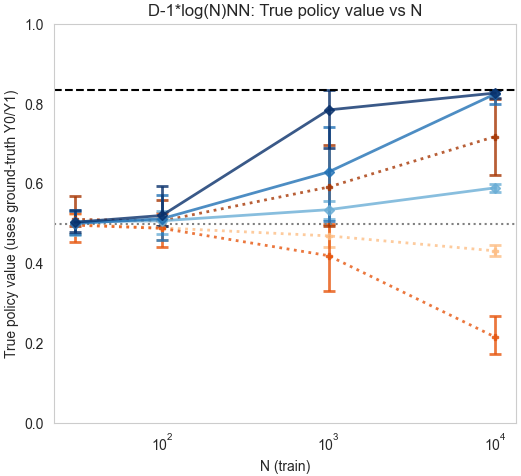}}

  \caption{Policy Learning with correlation between potential outcomes in the synthetic, heterogeneous RCT setting. (a–c) Linear Regression; (d–f) DKNN.   \textbf{Legends} as in \Cref{fig:synthetic_policy_learning_heterogeneous_RCT_uncorr}.}
  \label{fig:policy_learning_correlation}
\end{figure*}

\begin{figure}
  \centering
  {%
    \includegraphics[width=0.32\linewidth]{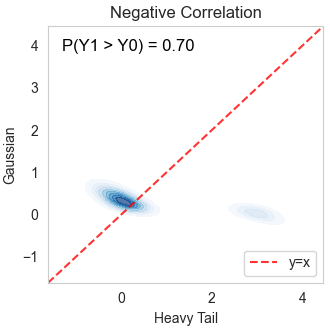}} 
    \hfill
  {%
    \includegraphics[width=0.32\linewidth]{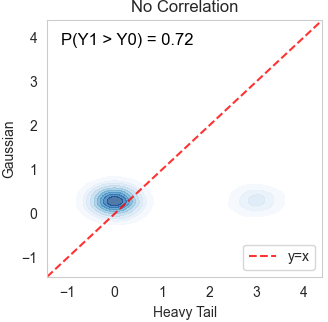}}
  {%
    \includegraphics[width=0.32\linewidth]{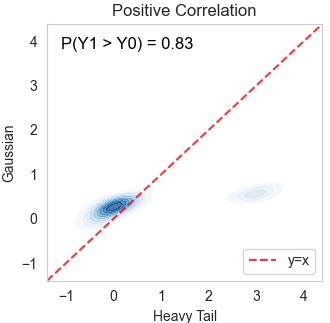}}
  \caption{Visualizing PNS, $P(Y1>Y0)$, with correlated potential outcomes $Y1$ and $Y0$. In our experiment, $Y1$ is Gaussian (y-axis) and $Y0$ has a heavy tail (x-axis). }
  \label{fig:potential_outcome_correlation}
\end{figure}

Policy learning is hardly affected. The visual in \Cref{fig:potential_outcome_correlation} shows how the correlation only affects a small part of the population (blue concentric circles) which cross the decision boundary (red line.) However, the policy value estimate is significantly effected. We show in \Cref{fig:policy_eval_correlated} policy evaluation using the correctly specified linear models. Under correlation, the policy values converge to the wrong policy value. They converge to the policy value without correlation.

\begin{figure}[t]
  \centering
  \subcaptionbox{DLR - negative correlation\label{fig:policy_eval_dlr_neg_corr}}{%
    \includegraphics[width=0.9\linewidth]{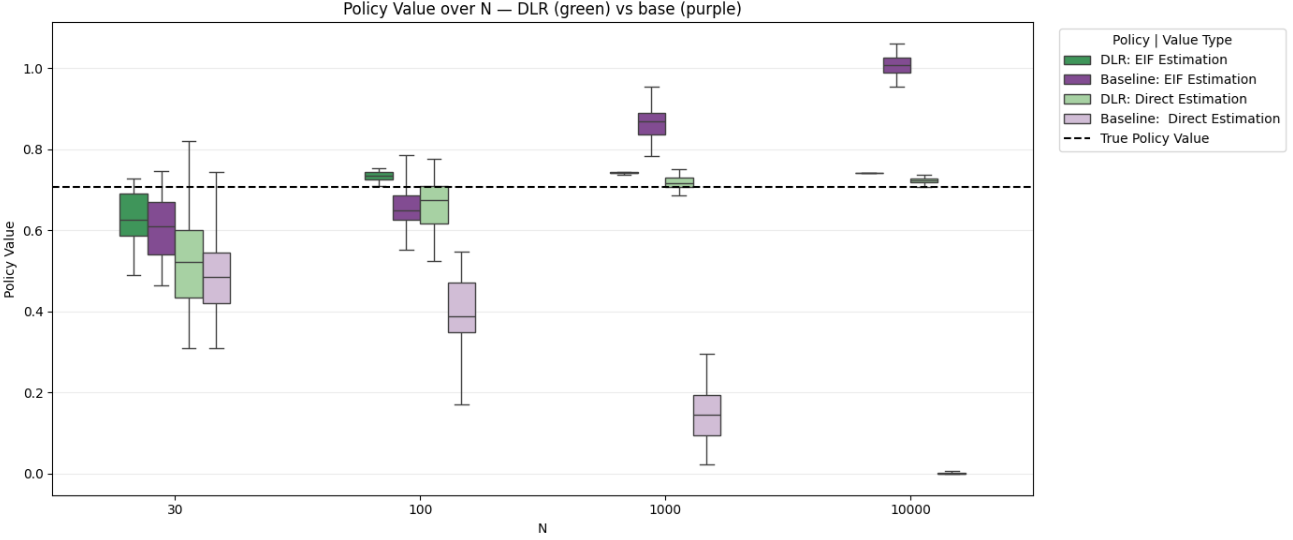}} \\[0.8em]
  \subcaptionbox{DLR\label{fig:policy_eval_dlr}}{%
    \includegraphics[width=0.9\linewidth]{policy_eval_dlr_hetero.png}} \\[0.8em]
  \subcaptionbox{DLR - positive correlation\label{fig:policy_eval_dlr_pos_corr}}{%
    \includegraphics[width=0.9\linewidth]{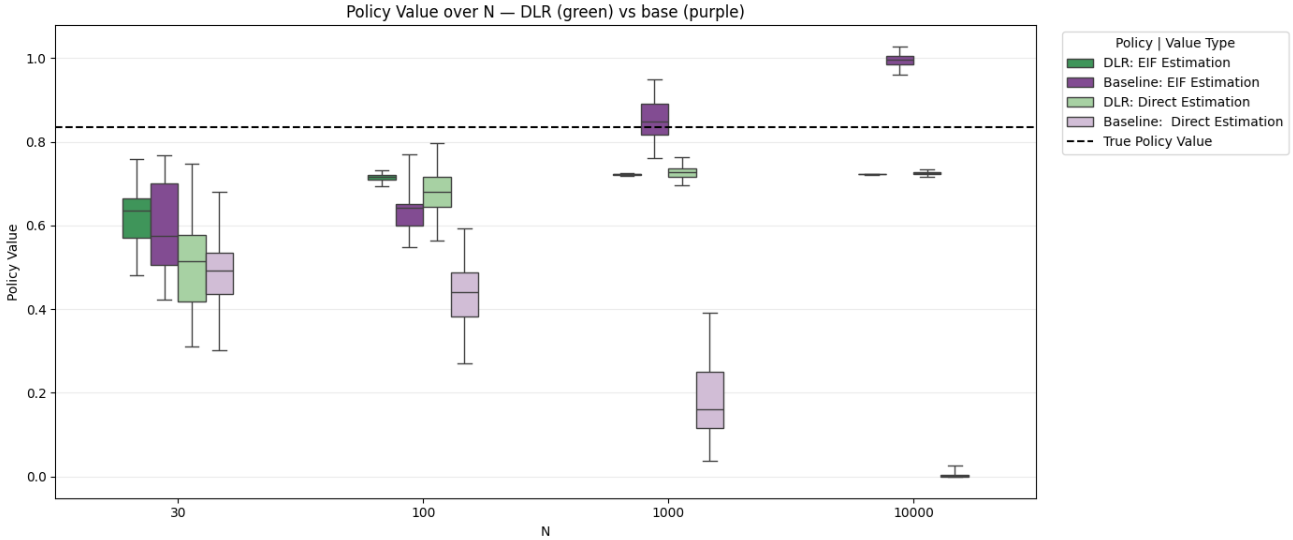}}

  \caption{Policy Evaluation of optimal policy with correlation in the semi-synthetic setting}
  \label{fig:policy_eval_correlated}
\end{figure}

\subsection{Semi-Synthetic Simulation details}
\label{app:semi_synthetic}
Validating causal-inference methods is difficult because the counterfactual ground-truth necessary for validation is never observed. This phenomena is called the \textit{fundamental challenge of causal inference} \citep{Holland1986}. The natural solution is to verify methods under ideal circumstances, such as random control trials. Random control trials provide unbiased estimates for CATE-estimates because they are unconfounded by construction. While CPTE-estimates also benefit from an RCT's inherent identifiability, they are still statistically problematic. Means can be reasonably estimated for small substratum. However, estimating the conditional distribution necessary for CPTE requires more data. This is especially true because the CPTE is particularly suited to tasks with complex conditional potential outcome distributions, as we have shown for PNS in the simulation above.

Therefore, to further validate our methods we turn to a semi-synthetic experiment. We use the Tennessee STAR study \citep{star1990state} for the starting point. The Tennessee Student/Teacher Achievement Ratio (STAR) study is a randomized experiment intended to measure the effect of class size on student test scores, namely math, reading and listening. Launched in 1985, the study followed students in grades K-3 recording their grades at the end of each academic year. Students were randomly assigned to one of three classroom types - small (13-17 students), regular (22-25) and regular with a full-time aide (22-25). Each school housed at least one classroom of each type.
\citep{DVN/SIWH9F_2008}. We apply our methods to the `win' function, a preference function (\cref{win_function_def}), which is often aggregated into various `win' statistics \citep{dong2023win}. The `win' function captures preference over hierarchical outcomes. \\
We work with a primary, binary outcome and a secondary, continuous outcome. This scenario is common in medical literature, where the primary outcome is often survival - binary or time-to-event - as the primary outcome and a secondary, continuous outcome such as blood pressure.  Our primary outcome is experiment attrition. We looked at kindergarten students found in the study and defined attrition as those students who are not found in the study in the first grade. Our secondary outcomes are the student's math scores. We define a `win' as students who are found in the study in first grade and have a high math score. \\
After removing students without a math score, we there are 1,762 student in small classrooms, 2,032 in regular classrooms and 2,077 in regular classrooms with full time aides. 454 students were removed due to a missing math score. We combine the two treatment arms with regular classrooms because the `win' function is only defined between two treatments. Although this may violate SUTVA, one of the study's findings was that the full-time aides were not significant \citep{DVN/SIWH9F_2008, krueger1999experimental}. 

As stated previously, calculating the conditional `win' function empirically over the RCT is statistically problematic, as many small substratum result in poorly estimated `win' functions despite being causally unbiased. Therefore, we create synthetic outcomes. We strive to replicate the original outcome distributions. We do so by training an MLP over the outcomes and then sampling from the MLP using MC-dropout \citep{gal2015dropout}, resulting in a Bayesian approximation of conditional distributions.

We then use these estimates of the conditional distributions as the ground-truth for the optimal policy and the optimal policy value. We evaluate our methods in a similar fashion to the synthetic experiments, training distributional and baseline models. We then use these models' predictions as plug-in estimators for the `win' function and predict the policy using the same policy learning methods as in the synthetic experiment. \\
In the synthetic experiments, we showed the advantages of a PNS-based policy over a mean-based policy for complex potential outcome distribution. The advantage of the `win' function is evident even without that complexity. Because the `win' function is composed of hierarchical outcomes, a successful `win'-based policy discerns between units for which the benefits of the secondary outcome outweigh the risks in the primary outcome. The baseline learners, which estimate the expected values of outcomes, will likely only prioritize the primary outcome. In the binary case, only when the estimated expectations are equal will the secondary outcome be considered.

We present our results on the semi-synthetic STAR data using the original features and our simulated outcomes. In contrast to the synthetic experiment, this more realistic simulation contains preferences close to the decision boundary, making policy learning more difficult.

\begin{figure*}[t]
  \centering
  \subcaptionbox{DLR\label{fig:dlr_semi}}%
    {\includegraphics[width=0.30\textwidth]{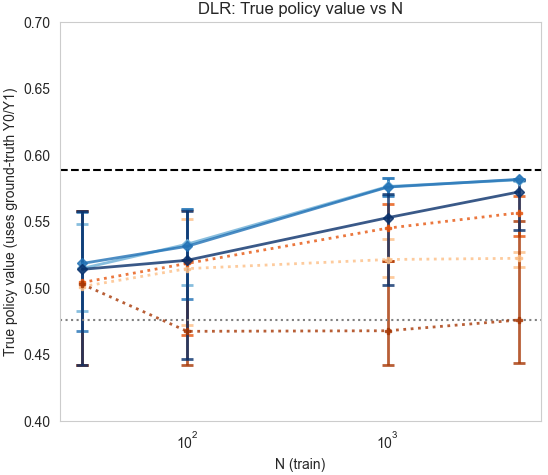}}
  \hfill
  \subcaptionbox{DKNN\label{fig:dknn_semi}}%
    {\includegraphics[width=0.30\textwidth]{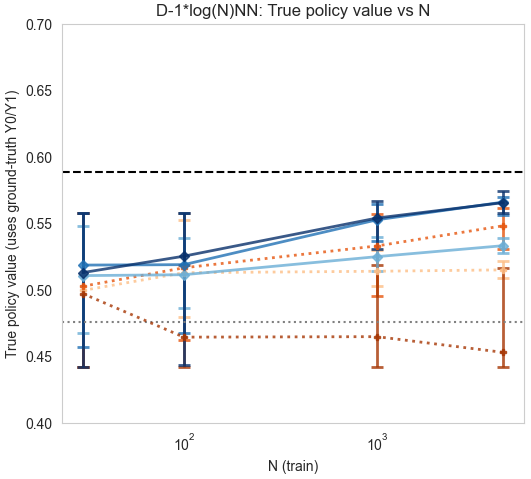}}
  \hfill
  \subcaptionbox{DRF\label{fig:drf_semi}}%
    {\includegraphics[width=0.30\textwidth]{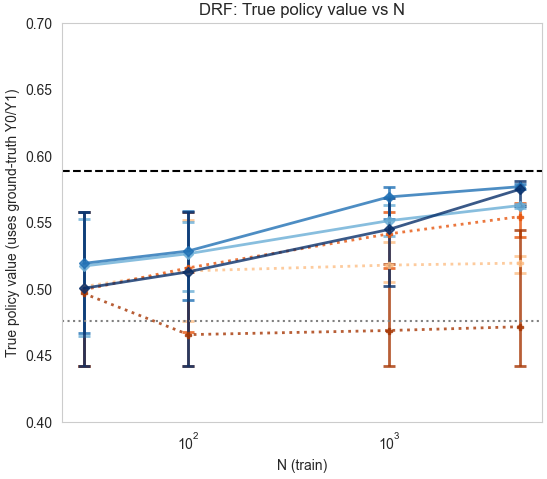}}

  \caption{Policy Learning with win policy value in the semi-synthetic setting.
    \textbf{Legends} as in \Cref{fig:synthetic_policy_learning_heterogeneous_RCT_uncorr}.}
\label{fig:semi_synthetic_policy_learning_grid}
\end{figure*}

We find in \cref{fig:semi_synthetic_policy_learning_grid} that as the size of our training data grows, the distributional methods best learn the optimal policy. In our graphs, our results with the oracle policy value (black) and by a random policy (gray). We also add a benchmark for the optimal primary outcome policy (brown), i.e. the policy value if we would base our policy only on the primary outcome, not the `win' preference function. 
We find that all of the baseline plug-in policies (yellow) do not even achieve the optimal primary outcome policy. The distributional policies (blues) and the direct policy value optimization policy for the baseline overtake the optimal primary outcome policy. These results indicate the power of distributional estimators to optimize CPTE functions. In contrast, the plug-in baseline policy barely beats an `preference' value of 0.5. This indicates the failure of CATE estimates performance on CPTE outcomes like `win' statistics.

We also experiment with correlated potential outcomes in the semi-synthetic setting. Our experiments show that our models and the baseline models are not sensitive to this violation of the assumptions for \Cref{lem:treatment_effect_modifiers}. See \cref{fig:semi_policy_learning_correlation}.

\subsubsection{Data Collection and Processing}
\label{app:semi_synth_data}

\textbf{Creating the cohorts}
We use the STAR dataset as the starting point for our semi-synthetic experiment as it appears in the AER R package \citep{AER_package}. We select students in kindergarten and create our binary outcome by checking who appears in the 1st grade cohort. Data is missing from the \textit{math} outcome (454 students). The missing \textit{math} outcome was dropped. We split the students into two intervention arms, small classrooms and regular classrooms. The latter arm includes two groups, regular classrooms without aides and those with full-time aides. Although, this may be a violation of SUTVA, the study's findings indicate that the full-time aides were not significant \citep{DVN/SIWH9F_2008, krueger1999experimental}. This resulted in two groups, the small classroom group (1,762 students) and the regular classroom group (4,109 students).

\textbf{Features}
Relevant features used to create the simulated outcomes and for training are: student gender, student ethnicity, teacher gender, teacher ethnicity, years of teaching experience, teacher's highest degree, where the teacher is found in their career (e.g. apprentice, probation, level 1, etc.), school location (rural/inner-city/urban/suburban) and whether the school offered free/reduced-fee lunches. The last variable is an indicator of low-income family background \citep{DVN/SIWH9F_2008}. All the variables are categorical with the exception of the years of teaching experience. Categorical variables were one-hot encoded and missing values were assigned zero to all categories. Missing values in years of experience were assigned the median.

\begin{table}[H]
\centering
\caption{Missing Variables}
\label{tab:missing_values}
\begin{tabular}{lcccc}
\toprule
Variable & Variable Type & Missing Values & Pre-processing \\
\midrule
Math                  & outcome & 454 & Removed \\
Student Ethnicity     & feature & 1   & Assigned zero in all categories \\
Free/Reduced-Fee Lunch& feature & 17  & Assigned zero in all categories \\
Degree                & feature & 21  & Assigned zero in all categories \\
Career Ladder         & feature & 541 & Assigned zero in all categories \\
Teacher Ethnicity     & feature & 60  & Assigned zero in all categories \\
Years of Teaching Experience& feature & 21  & Median \\
\bottomrule
\end{tabular}
\end{table}

\textbf{Outcome Simulation}
An MLP was trained for each intervention arm and outcome. We summarize the hyperparameters and evaluation metrics of each model. Dropout of 10\% was used for all models in training

\begin{table}[H]
\centering
\caption{Summary of deep model models}
\label{tab:deep_models}
\begin{tabular}{lccccccc}
\toprule
Model & Hidden Dimensions & Epochs & Batch Size & Learning Rate & Brier Score & RMSE (std) \\
\midrule
Small -- Dropout      & [64, 64]       & 35  & 128 & $5\times 10^{-4}$ & 0.17 & -- \\
Regular -- Dropout    & [64, 64]       & 34  & 128 & $5\times 10^{-4}$ & 0.19 & -- \\
Small -- Math Score   & [64, 64, 64]   & 100 & 32  & $1\times 10^{-4}$ & -- & 48.5 (49.5) \\
Regular -- Math Score & [64, 64, 64]   & 100 & 32  & $1\times 10^{-4}$ & -- & 45.1 (46.7)\\
\bottomrule
\end{tabular}
\end{table}

After training the models, we sample using MC-dropout \citep{gal2015dropout} with a drop-out rate of 5\%. For each unit, we sample 1000 potential outcomes for each intervention arm. These simulations induce a difficult, non-trivial policy learning problem. Most units' simulated `win' preference function are between 0.4 and 0.6. We compare the optimal policies for the primary and secondary outcome:

\begin{table}[H]
\centering
\caption{Comparing Optimal Policies}
\label{tab:semi_compare_optimal_policies}
\begin{tabular}{lccc}
\toprule
Policy & Treat w/ smaller & Agreement with Win Policy & Win Policy Value\\
\midrule
Win (Dropout, Math) Policy & 77\% & 100\% & 0.59 \\
Dropout Policy & 67\% & 70\% & 0.54\\
Math Score Policy & 75\% & 92\% & 0.58\\
\bottomrule
\end{tabular}
\end{table}

As can be seen from \cref{tab:semi_compare_optimal_policies}, the policy prioritizing the secondary outcome is close to the optimal policy. This is due to the lack of significant treatment effect on the binary outcome. However, this poses a difficulty to learning the optimal policy in the `win' setting. The hierarchical ordering of `winning' forces the policy to first take into consideration the primary outcome.

\textbf{Training \& Evaluation} The training and evaluation are similar to the synthetic simulation. The primary difference is the inclusion of a binary, primary outcome. Binary outcomes are completely parameterized by their Bernoulli parameter, which is the expected value. There is no need for our distributional methods on the primary outcome. Therefore, we perform model selection on the training data and use the same model for all distributional and baseline methods. We found a KNN classifier with 11 neighbors to offer the best brier score (0.17). Our sample size is also limited by the size of the kindergarten cohort. We divide our data into a hold-out set containing 1,500 students (25\%) and our pool of training data is the remaining 4,403 students. We randomly sample training data with $N\in [30, 100, 1000, 4403]$ and train the estimation models with their derivative policy learning models. We repeat this process fifty times for each model and $N$.

\subsubsection{Correlation} \label{app:semi_synth_correlation}
We also explore the impact of correlation between potential outcomes in our semi-synthetic study. To do so, we employ the Iman-Conover method \citep{iman1982distribution} commonly used to induce correlation between variables in simulations \citep{usesIman1, usesIman2, usesIman3}. This method is distribution free and preserves the marginal distributions. Instead of changing the values, it reorders them in such a way that induces an approximation of the desired correlation. We use this method to induce correlation between the potential outcomes of each outcome. Therefore, the marginal potential outcome distributions are identical between the three experiments. For each experiment, we take a single sample from the potential outcomes of each unit and assign the observed outcome based on the unit's treatment arm.

We evaluate our models when there is positive or negative correlation between the potential outcomes, results are available in \Cref{tab:semi_compare_correlation}. For the positive correlation, the desired correlation was set at 0.5 for each outcome. The induced correlation was 0.28 for the binary potential outcomes and 0.47 for the continuous potential outcomes. For the negative correlation, the desired correlation was set at -0.5 for each outcome. The induced correlation was -0.22 for the binary potential outcomes and -0.41 for the continual potential outcomes.

\begin{table}[H]
\centering
\caption{Comparing Optimal Policies}
\label{tab:semi_compare_correlation}
\begin{tabular}{lcccc}
\toprule
Correlation (Induced) & Policy & Treat w/ smaller & Agreement with Win Policy & Win Policy Value\\
\midrule
     & Win (Dropout, Math) Policy & 73\% & 100\% & 0.59 \\
-0.5 (-0.22/-0.41) & Dropout Policy & 67\% & 62\% & 0.53\\
     & Math Score Policy & 75\% & 98\% & 0.59\\
\midrule
     & Win (Dropout, Math) Policy & 77\% & 100\% & 0.59 \\
  0  & Dropout Policy & 67\% & 70\% & 0.54\\
     & Math Score Policy & 75\% & 92\% & 0.58\\
\midrule
     & Win (Dropout, Math) Policy & 74\% & 100\% & 0.65 \\
0.5 (0.28/0.47) & Dropout Policy & 67\% & 63\% & 0.55\\
     & Math Score Policy & 75\% & 99\% & 0.65\\
\bottomrule
\end{tabular}
\end{table}

Similarly to the synthetic experiment, the correlations change the optimal policy slightly (8\%) and the positive correlation raises the value of the optimal policy value. We show in \cref{fig:semi_policy_learning_correlation} that our models are only slightly affected. Because the positive correlation causes the units `win' functions to be more extreme, we notice that as the models better learn the correct policy they gain significant policy value.

\begin{figure*}[t]
  \centering
  \subcaptionbox{DLR: -0.5\label{fig:dlr_semi_neg}}%
    {\includegraphics[width=0.30\textwidth]{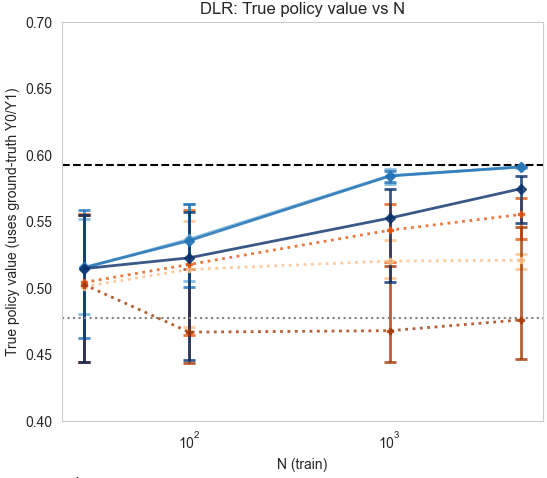}}
  \hfill
  \subcaptionbox{DLR: 0\label{fig:dlr_semi}}%
    {\includegraphics[width=0.30\textwidth]{dlr_semi.png}}
  \hfill
  \subcaptionbox{DLR: 0.5\label{fig:dlr_semi_pos}}%
    {\includegraphics[width=0.30\textwidth]{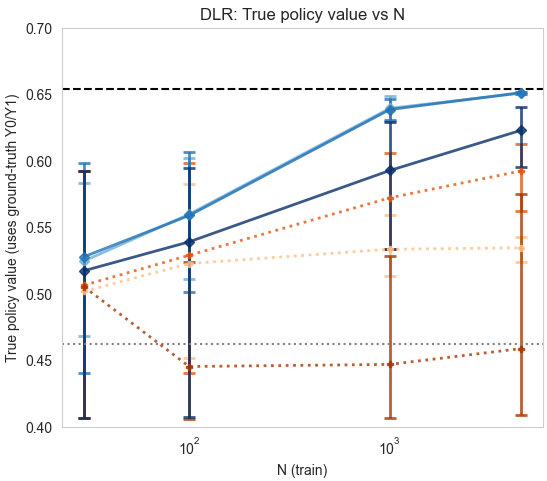}}

  \medskip
  \subcaptionbox{DKNN: -0.5\label{fig:dknn_semi_neg}}%
    {\includegraphics[width=0.30\textwidth]{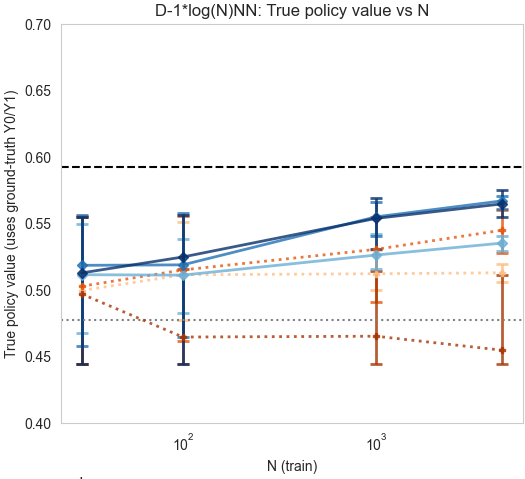}}
  \hfill
  \subcaptionbox{DKNN: 0\label{fig:dknn_semi}}%
    {\includegraphics[width=0.30\textwidth]{dknn_semi.png}}
  \hfill
  \subcaptionbox{DKNN: 0.5\label{fig:dknn_semi_pos}}%
    {\includegraphics[width=0.30\textwidth]{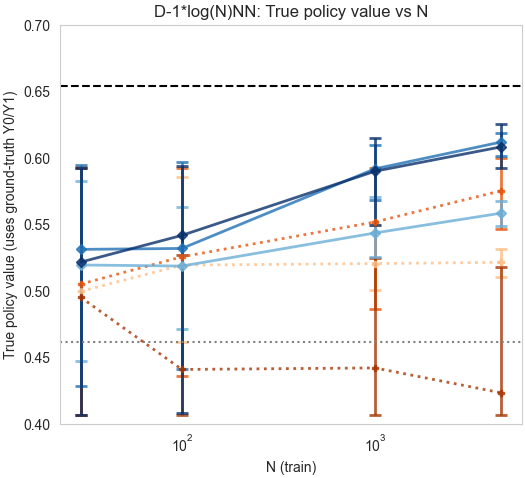}}

    \medskip
  \subcaptionbox{DRF: -0.5\label{fig:drf_semi_neg}}%
    {\includegraphics[width=0.30\textwidth]{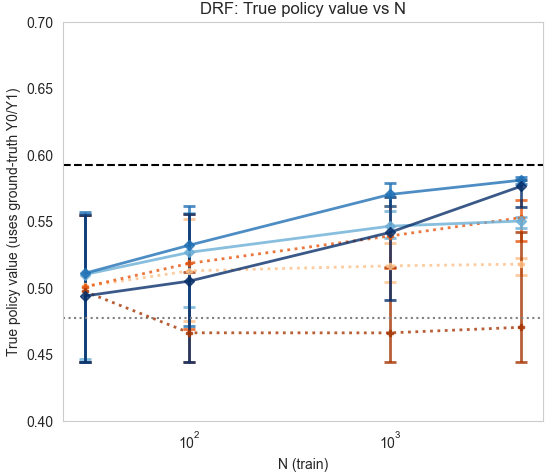}}
  \hfill
  \subcaptionbox{DRF: 0\label{fig:drf_semi}}%
    {\includegraphics[width=0.30\textwidth]{drf_semi.png}}
  \hfill
  \subcaptionbox{DRF: 0.5\label{fig:drf_semi_pos}}%
    {\includegraphics[width=0.30\textwidth]{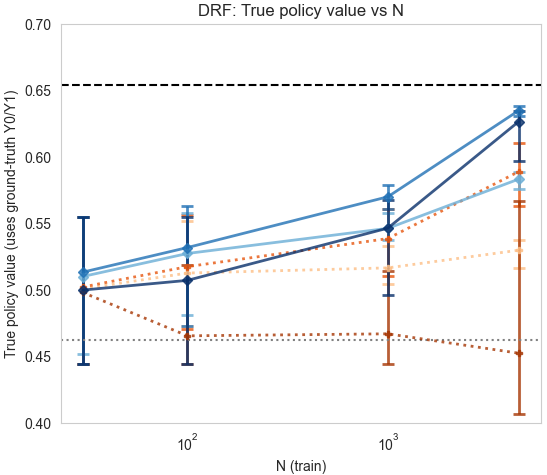}}

  \caption{Policy Learning with correlation between potential outcomes in the semi-synthetic setting. (a–c) Linear Regression; (d–f) DKNN; (g–i) DRF.
    \textbf{Legends} as in \Cref{fig:synthetic_policy_learning_heterogeneous_RCT_uncorr}.}
  \label{fig:semi_policy_learning_correlation}
\end{figure*}

\subsubsection{Impact of the number of neighbors in KNN}
\label{app:nb_neighbors_knn}

\Cref{fig:semi_synth_neighbors} shows the impact of adding more neighbors for policy learning.
As predicted by our theory, the plug-in OTR (\Cref{sec:plug_in_OTR}) requires $k\to\infty$ to converge. Indeed, for $k=1$ the plug-in method performs as poorly as the baseline plug-in (almost as poorly as random guesses), while increasing $k$ leads to higher policy values, and faster with higher $k$.
However, and quite interestingly, this is not the case for policy value optimization methods (\Cref{sec:plug_in_value,sec:EIF_correction}) due to the fact that values are averaged over the whole population. In fact, adding more neighbors does not seem to impact much convergence.

\begin{figure*}[t]
  \centering
  
    \subcaptionbox{$k=1$ neighbor \label{fig:neighbors_k_1}}%
    {\includegraphics[width=0.6\textwidth]{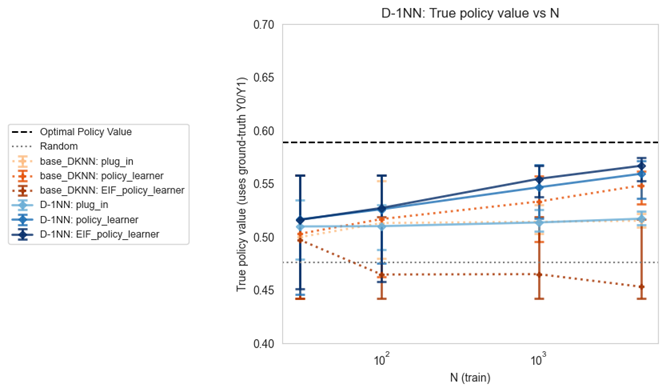}}
  \hfill
  \subcaptionbox{$k=\log(n)$ neighbors\label{fig:neighbors_k_logn}}%
    {\includegraphics[width=0.6\textwidth]{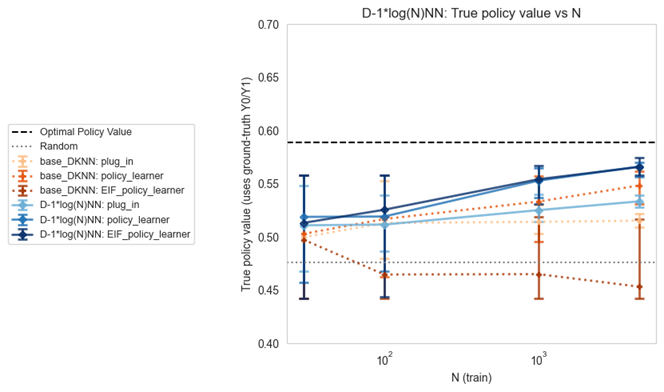}}
  \hfill
  \subcaptionbox{$k=2\log(n)$ neighbors\label{fig:neighbors_k_logn}}%
    {\includegraphics[width=0.6\textwidth]{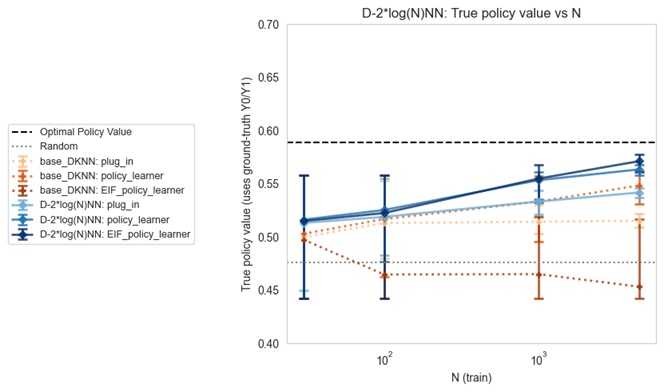}}

  \caption{$k-$NN with $k=1$, $k=\log(n)$ and $k=2\log(n)$ on the semi-synthetic experiment without correlation}
  \label{fig:semi_synth_neighbors}
\end{figure*}

\subsubsection{Policy Agreement - Optimizing `Win' vs. Primary outcome vs. Secondary outcome}
\label{app:semi_synth_policy_agreements_with_win_mean1_mean2}

Lastly, we compare the \textit{policy agreement} \textit{i.e.}, the percentage of agreement between our learned policies and a given oracle policy that is either \textit{(i)} the optimal policy learned with oracle information using the hierarchical outcomes, \textit{(ii)} the optimal policy learned with oracle information that only optimizes the primary outcome, and \textit{(iii)} the optimal policy learned with oracle information that only optimizes the secondary outcome.
Results are reported in \Cref{fig:semi_synth_policy_agreements_hierarchical}.
We show results for the DLR method. These results are also consistent for KNN and DRF.

\begin{figure*}[t]
  \centering
  
    \subcaptionbox{Agreement with oracle policy optimizing by win}%
    {\includegraphics[width=0.6\textwidth]{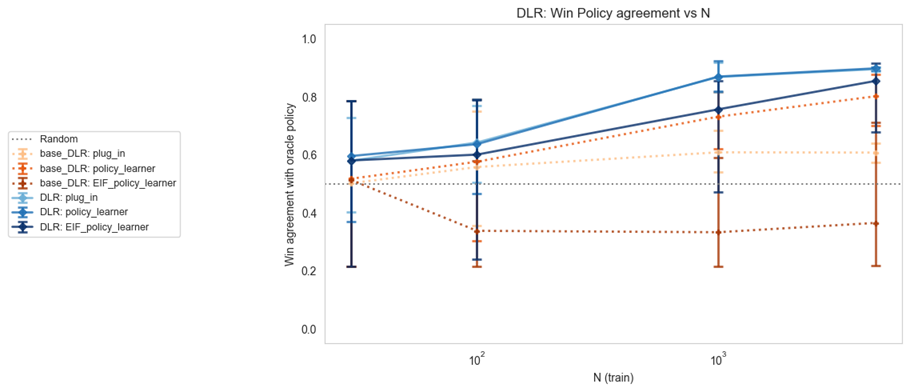}}
  \hfill
  \subcaptionbox{Agreement with oracle policy optimizing by primary outcome}%
    {\includegraphics[width=0.6\textwidth]{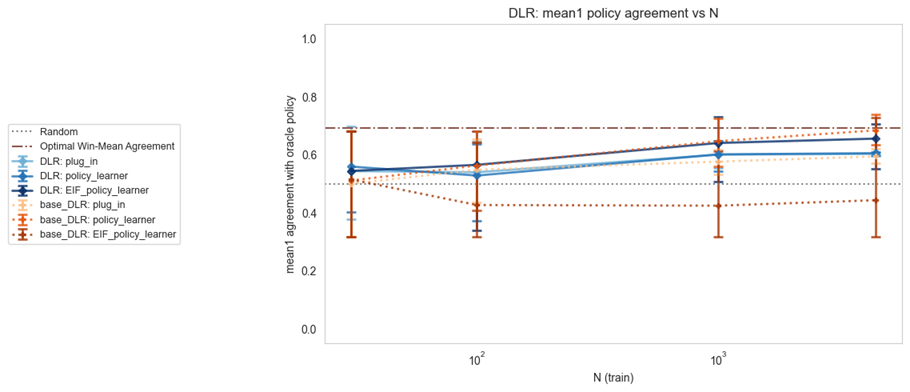}}
  \hfill
  \subcaptionbox{Agreement with oracle policy optimizing by secondary outcome}%
    {\includegraphics[width=0.6\textwidth]{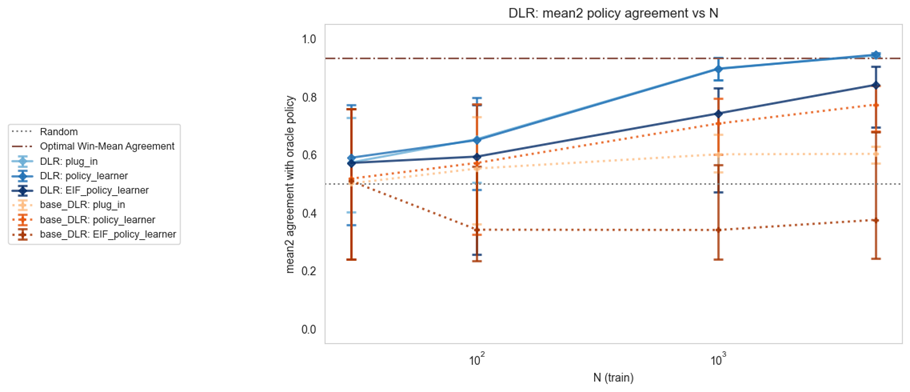}}

  \caption{Learned Policy's agreement with different oracle policies}
  \label{fig:semi_synth_policy_agreements_hierarchical}
\end{figure*}

\end{document}